\documentclass{article}

\usepackage[preprint,nonatbib]{neurips_2022}

\usepackage[utf8]{inputenc} 
\usepackage[T1]{fontenc}    
\usepackage{hyperref}       
\usepackage{url}            
\usepackage{booktabs}       
\usepackage{amsfonts}       
\usepackage{nicefrac}       
\usepackage{microtype}      
\usepackage{xcolor}    

\usepackage{amsmath,amssymb,amsthm,commath,esint,tikz-cd,tikz, mathtools, physics}
\usetikzlibrary{shapes,arrows,positioning}
\usepackage[mathscr]{euscript}
\usepackage{graphicx}
\usepackage{float}
\usepackage{fancyhdr}
\usepackage{bm}
\usepackage{bbm}
\usepackage[ruled,vlined]{algorithm2e}

\numberwithin{equation}{section}

\usepackage{amsthm,amssymb}
\usepackage{mathtools,thmtools}
\usepackage{bm,esint}
\usepackage{color}
\usepackage{float,graphicx,subcaption}
\usepackage{cancel}
\usepackage[shortlabels]{enumitem}
\usepackage{mathrsfs}  
\usepackage{bbm}
\usepackage{euscript}
\usepackage{hyperref}
\usepackage{cleveref}
\usepackage{upgreek}

\usepackage{todonotes}

\newcommand{\size}{\mathrm{size}}

\newcommand{\fb}{\bm{\mathrm{e}}}
\newcommand{\fbhat}{\widehat{\bm{\mathrm{e}}}}
\newcommand{\Id}{\mathrm{Id}}

\newcommand{\Err}{\widehat{\mathscr{E}}}

\newcommand{\slot}{{\,\cdot\,}}

\newcommand{\T}{\mathbb{T}}
\newcommand{\R}{\mathbb{R}}

\newcommand{\N}{\mathbb{N}}

\newcommand{\Z}{\mathbb{Z}}

\newcommand{\G}{\mathcal{G}}
\renewcommand{\div}{{\mathrm{div}}}
\newcommand{\Lip}{\mathrm{Lip}}

\newcommand{\cA}{\mathcal{A}}
\newcommand{\cB}{\mathcal{B}}

\newcommand{\cD}{\mathcal{D}}
\newcommand{\cE}{\mathcal{E}}
\newcommand{\cF}{\mathcal{F}}
\newcommand{\cG}{\mathcal{G}}

\newcommand{\cI}{\mathcal{I}}
\newcommand{\cJ}{\mathcal{J}}
\newcommand{\cK}{\mathcal{K}}
\newcommand{\cL}{\mathcal{L}}

\newcommand{\cN}{\mathcal{N}}

\newcommand{\cP}{\mathcal{P}}

\newcommand{\cR}{\mathcal{R}}

\newcommand{\cV}{\mathcal{V}}

\newcommand{\cX}{\mathcal{X}}
\newcommand{\cY}{\mathcal{Y}}

\newcommand{\hg}{\widehat{g}}

\newcommand{\outspace}{\mathcal{Y}}

\newcommand{\chat}{{\widehat{c}}}

\renewcommand{\tr}{{{\tau}}}
\newcommand{\trunk}{{\bm{\tau}}}
\newcommand{\branch}{{\bm{\beta}}}

\newcommand{\enum}{{\upkappa}}

\renewcommand{\epsilon}{\varepsilon}

\newcommand{\Uhat}{\widehat{U}}

\newcommand{\s}{s} 

\newcommand{\bigO}{\mathcal{O}}
\newcommand{\E}[1]{{\mathbb{E}\left[ #1 \right]}} 

\newtheorem{theorem}{Theorem}[section]

\newtheorem{remark}[theorem]{Remark}
\newtheorem{definition}[theorem]{Definition}
\newtheorem{lemma}[theorem]{Lemma}

\newtheorem{corollary}[theorem]{Corollary}

\title{Variable-Input Deep Operator Networks}

\author{
Michael Prasthofer\thanks{Equal contribution.} \hspace{0.5cm} Tim De Ryck$^*$  \hspace{0.5cm}Siddhartha Mishra\\
Seminar for Applied Mathematics\\
ETH Zürich, Switzerland 
}

\begin{document}

\maketitle

\begin{abstract}
Existing architectures for operator learning require that the number and locations of sensors (where the input functions are evaluated) remain the same across all training and test samples, significantly restricting the range of their applicability. We address this issue by proposing a novel operator learning framework, termed Variable-Input Deep Operator Network (VIDON), which allows for random sensors whose number and locations can vary across samples. VIDON is invariant to permutations of sensor locations and is proved to be universal in approximating a class of continuous operators. We also prove that VIDON can efficiently approximate operators arising in PDEs. Numerical experiments with a diverse set of PDEs are presented to illustrate the robust performance of VIDON in learning operators. 
\end{abstract}

\section{Introduction}
Operators are mappings between infinite-dimensional spaces. They arise in a large variety of contexts in science and engineering, particularly when the underlying models are ordinary (ODEs) or partial (PDEs) differential equations. A prototypical example for operators is provided by the so-called solution or evolution operator of a time-dependent PDE, which maps an input (infinite-dimensional) function space of initial conditions to an output function space of solutions of the PDE at certain point of time. Given the ubiquity of ODEs and PDEs in applications, \emph{learning operators} from data is of great significance in science and engineering, \cite{HIG,NO} and references therein. 

As inputs and outputs of operators are infinite-dimensional (e.g. functions, infinite sequences), conventional neural networks cannot be directly deployed to learn them. Instead, a new field of \emph{operator learning} is rapidly emerging, wherein one designs novel learning architectures to approximate such operators. A popular operator learning paradigm is that of \emph{neural operators} \cite{NO}, which generalize the structure of neural networks, wherein each hidden layer consists of a \emph{non-local} affine operator, composed with a local (scalar) non-linear activation function. Reflecting the infinite-dimensional structure of the underlying learning task, the non-local affine layer amounts to integrating with respect to a kernel and choosing different kernels leads to graph kernel operators, \cite{GKO}, low-rank kernel operators \cite{NO} and multipole expansions \cite{Mpole}. Evaluating a convolution-based kernel efficiently in Fourier space (via FFT), yields the \emph{Fourier neural operator} (FNO) \cite{FNO}, which has been rigorously proved to be \emph{universal}, efficient in learning operators arising in PDEs \cite{kovachki2021universal} as well as being very successfully employed in a variety of applications in science and engineering \cite{FNO,FNO1,FNO2} and references therein. However, FNOs are restricted to operators where the discretization of the underlying domains is (or can be efficiently mapped to) a Cartesian grid, considerably limiting the range of their applicability \cite{fair}. 

An alternative framework is that of \emph{operator networks} \cite{ChenChen1995} and their deep version, \emph{DeepONets} \cite{deeponets}. This architecture is based on two different sets of neural networks, so-called \emph{trunk nets} which span the infinite-dimensional output space and \emph{branch nets} which map the (encoded) input into the coefficients of the trunk nets. DeepONets are also universal, efficiently approximate operators arising in PDEs \cite{LMK1} and are widely used in scientific computing \cite{deeponets,donet1,donet2,donet3} and references therein. In contrast to FNOs, DeepONets can handle learning operators on very general domains and boundary conditions. However, DeepONets have their own set of limitations. In particular, the (infinite-dimensional) input function to the branch net of a DeepONet has to be projected to finite dimensions by evaluating it on a finite set of \emph{sensors}, located in the underlying domain. Although these sensor locations can be \emph{randomly chosen} inside the domain \cite{deeponets,LMK1}, the number and location of these sensor points has to be invariant across all training (and test) samples. Similarly,  FNOs require input sensors to be located on a Cartesian grid for each sample.  

On the other hand, the training (and test) data for operator learning is generated either from physical measurements (observations) or computer simulations (or a combination of them). In both scenarios, it is too restrictive to expect that data is available at the same set (or number) of points across all training samples. For instance, different data sources (measurement devices) can be placed at different locations at different time periods or for different domains (experimental conditions), within the same data set. Moreover, measurements at some sensor locations could be missing due to device faults. Additionally, there will always be an intrinsic uncertainty in the exact location of measurement devices. Similarly, numerical simulations at different spatio-temporal resolutions will be combined in the same training and test data set. Hence, the lack of flexibility in sensor locations for \emph{encoding} the inputs to DeepONets and FNOs constitutes a major limitation for current operator learning frameworks (see Figure \ref{fig:1}).   

This limitation of fixed number and location of sensors points to a fundamental issue with existing operator learning frameworks as one can  view their inputs as vectors of fixed length containing function values at fixed locations, rather than functions which could be evaluated at arbitrary points of the domain. This raises questions on the very essence of operator learning with current architectures.  A related issue pertains to the notion of \emph{permutation invariance} i.e., permuting input sensor locations should not alter the output of an operator learning framework as the same underlying input function is being sampled. Thus, it is imperative to require that operator learning frameworks be permutation-invariant.

The above considerations set the stage for the current paper where we propose a novel architecture for operator learning that allows for variable (flexible) sensor locations. To this end and motivated by the permutation-invariant deep learning frameworks such as \emph{deep sets} \cite{zaheer2017deep,wagstaff2019limitations} as well as \emph{transformers} \cite{vaswani2017attention}, we propose an operator learning framework termed \emph{Variable-Input Deep Operator Networks} (VIDON) in Section \ref{sec:2}, that allows for random locations for input sensors in each sample as well as for the number and locations of sensors to vary across samples (see Figure \ref{fig:1}). We prove in Section \ref{sec:analysis} that VIDON is \emph{universal} in approximating continuous operators that map into Sobolev spaces. Moreover, we also prove that VIDON can efficiently approximate operators stemming from a variety of PDEs, by showing that the size of the underlying neural networks only grows polynomially (at worst) in the inverse of the accuracy. Finally in Section \ref{sec:4}, we illustrate VIDON in a series of numerical experiments for different PDEs, showing that it can accurately approximate the underlying operators, while presented with very different input sensor configurations. The notation and technical details for proofs and implementations are presented in the supplementary material ({\bf SM}).
\begin{figure}[ht!]
\centering
\includegraphics[width=\linewidth]{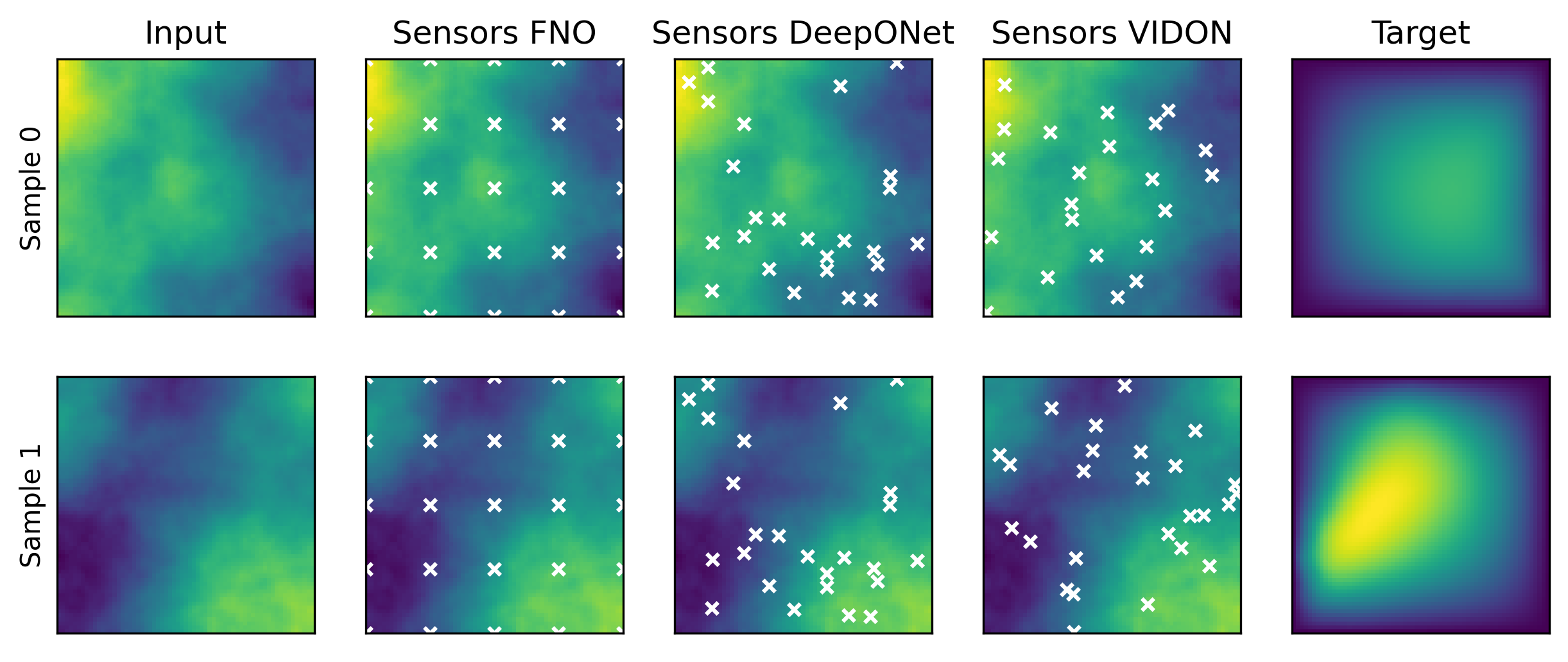}
\caption{Learning the operator mapping the permeability coefficient $a$ (Input) to the solution $u$ (Target) for Darcy flow \eqref{eq:darcy}, for two different realizations (samples) of the input. The input has to be evaluated at \emph{sensor} points. FNO requires a Cartesian grid of sensors for each sample whereas as DeepONets allow for a random cloud of sensors. However, the number and location of sensors has to be the same for all training and test samples. In contrast, VIDON \eqref{eq:tb-deeponet} allows for random sensor points, whose location and number can vary for each sample.}
\label{fig:1}
\end{figure}

\section{Variable-Input Deep Operator Networks}
\label{sec:2}
\paragraph{Setting.} For compact sets $D,U \subset \R^d$, we consider (general forms of) operators $\cG:\cX\to\cY$, where $\cX\subset L^2(D;\R^{d_v})$ and $\cY\subset L^2(U; \R^{d_u})$. In the context of time-dependent PDEs of the general abstract form, $\cL_a(u)=0$ with $u(0,\cdot)=u_0$, where $a$ is a parameter function, we will consider operators that map a parameter function $a$ to the solution $u$ i.e., $\cG:\cX\to\cY:a\mapsto u$, as well as operators that map the initial condition $u_0$ to the solution $u$ i.e., $\cG:\cX\to\cY:u_0\mapsto u$. 

\paragraph{DeepONets.} Following \cite{deeponets,LMK1}, one approximates the operators $\cG$ by starting with an \emph{encoder} $\cE:\cX\to \R^m$ that maps every input function $u$ to $m$ fixed \emph{sensors} and its sensor values $((x_j,u(x_j)))_{1 \leq j \leq m}$. Two sets of neural networks are then defined, a \emph{branch net} $\branch^*: \R^m \to \R^{p}$ and a \emph{trunk net} $\trunk :U\to\R^{p+1}$.
The branch and trunk nets are then combined to approximate the underlying non-linear operator as the \emph{DeepONet} $\cN^*$,
\begin{equation}
    \label{eq:deeponet}
\cN^*: \cE(\cX)\subset \R^m \to\outspace : \cE(u) \mapsto \tr_0(y)+\sum_{k=1}^p \beta_k^*(\cE(u)) \tr_k(y).
\end{equation}

\paragraph{Functions on sets.} In DeepONets \eqref{eq:deeponet}, the input to the branch net is a tuple of fixed length, whose entries are function values of the input function $u$ at $m$ fixed sensors. As described in the introduction, this is necessarily limiting.  To allow for variable sensor locations as well as invariance of the output to permutations of sensor locations, the input to the branch net should consist of sets of variable size, rather than tuples of fixed length. To this end,  we follow \cite{wagstaff2019limitations} to define $\mathfrak{X}^{\leq M}$ to be the set of subsets of a set $\mathfrak{X}$ containing at most $M\in\N$ elements and we denote by $\mathfrak{X}^{\cF}$ the set of finite subsets of $\mathfrak{X}$. For $M\in\N$, $u\in \cX$ and (possibly random) sensor points $x_1(u), \ldots, x_M(u)$ we can then define an encoder $\cE$ by, 
\begin{equation}\label{eq:encoder}
\begin{split}
    &\cE: \cX\times \{1, \ldots, M\}\to (\R^{d+d_v})^{\leq M}: (u, m) \to \cE_m(u), \quad \text{where},\\
    &\cE_m:\cX\to (\R^{d+d_v})^{m}: u\mapsto \{(x_j,u(x_j))\}_{j=1}^m, \quad 1\leq m\leq M.
\end{split}
\end{equation}
Given this encoding, the branch net in \eqref{eq:deeponet} needs to be a permutation-invariant function that allows sets with variable size as input, i.e., $\branch:(\R^{d+d_v})^{\leq M}\to \R^p$. Based on the result of \cite{zaheer2017deep} on permutation-invariant functions, it has been proven in \cite[Theorem 4.1]{wagstaff2019limitations} that every continuous function $f:\R^{\leq M}\to \R$ is necessarily \emph{continuously sum-decomposable via $\R^M$}, meaning that it must be of the form $f(X) = \rho\left(\sum_{x\in X}\varphi(x)\right)$, $X\in \R^{\leq M}$, where $\rho:\R^M\to \R$ and $\varphi: \R\to \R^M$ are continuous functions. This characterization constitutes the starting point of our new architecture, which consists of the following ingredients,

\paragraph{Input encoding.}

The input is a function $u\in \cX$ that is sampled at $m=m(u)$ sensor points by the encoder $\cE_m$ \eqref{eq:encoder} i.e., $\cE_{m}(u) = \{(x_j, u(x_j))\}_{j=1}^{m}$ where $\{ x_j \}_{j=1}^{m}\subset D$ are the sensor coordinates and $\{ u(x_j) \}_{j=1}^{m}\subset \mathbb{R}^{d_v}$ the corresponding sensor values. Note that the sampling procedure, both in terms of number of sensors and their locations, is allowed to be different for every input. The samples are then further encoded as follows, 
\begin{equation}\label{eq:psi}
    \Psi : \R^{d+ d_v}\to \R^{d_{enc}}: (x_j, u(x_j)) \mapsto \Psi_c\left(x_j\right) + \Psi_v\left(u(x_j)\right) =: \psi_j
\end{equation}
where $\Psi_c:\R^{d}\to \R^{d_{enc}}$ is the \emph{coordinate encoder} and $\Psi_v:\R^{d_v}\to \R^{d_{enc}}$ the \emph{value encoder}, both modeled as trainable multilayer perceptrons (MLPs). 
\paragraph{Head(s).} The well-known \emph{attention} mechanism of the transformer \cite{vaswani2017attention} architecture transforms sequential inputs into weighted values, where the (convex) weights are computed in terms of \emph{correlations} between inputs at different locations. Loosely motivated by this mechanism but seeking to keep computational complexity linear (instead of quadratic as in the case of attention of \cite{vaswani2017attention}) in inputs, we process the encoded inputs $\{\psi_j\}_{j=1}^m$ in the following manner. First for any $1 \leq \ell \leq H$, the \emph{values} of the \emph{head} $\ell$ are calculated using a single MLP $\Tilde{\nu}^{(\ell)}: \R^{d_{enc}} \to \R^p$. Next, the corresponding weights are computed as,
\begin{equation}\label{eq:omega}
     \omega^{(\ell)}: \R^{d_{enc}} \to \R: \psi_j \mapsto \frac{\exp(\Tilde{\omega}^{(\ell)}(\psi_j) / \sqrt{d_{enc}})}{\sum_{k=1}^{m} \exp(\Tilde{\omega}^{(\ell)}(\psi_k) / \sqrt{d_{enc}})}, \qquad \text{where}\quad  \Tilde{\omega}^{(\ell)}: \R^{d_{enc}} \to \R 
\end{equation}
is instantiated as an MLP. The output of a single head with index $\ell$ is then given by, 
\begin{equation}\label{eq:nu}
    \nu^{(\ell)}: \R^{d_{enc}}\to \R^p: \Psi(\cE_m(u)) = (\psi_j)_{j=1}^m \mapsto \sum_{j=1}^m \omega^{(\ell)}(\psi_j)\Tilde{\nu}^{(\ell)}(\psi_j). 
\end{equation}

Note that $\nu^{\ell}$ is permutation-invariant and well-defined for any $m\in\N$. Next, in analogy to multihead attention \cite{vaswani2017attention}, we concatenate the multiple heads $\nu^{(\ell)}$ and denote the result by $\nu = [\nu^{(1)},\:\ldots,\:\nu^{(H)}]$. 

\paragraph{Variable-Input Deep Operator Network.} Finally, we combine the outputs of the multiple heads using another MLP $\Phi : \R^{H\cdot p}\to \R^p$. The Variable-Input Deep Operator Network (VIDON) is then defined by replacing the branch net $\branch^*$ in \eqref{eq:deeponet} by $\branch := \Phi \circ \nu \circ \Psi$ i.e., for any $m\in \N$ we define,
\begin{equation}\label{eq:tb-deeponet}
    \cN : \cE(\cX) \to \cY : \cE_m(u) \mapsto \tr_0(y)+\sum_{k=1}^p \beta_k(\cE_m(u)) \tr_k(y).
\end{equation}
We measure the size of VIDON, $\size(\cN)$, by the total number of unique parameters in both the branch and trunk nets, which is independent of the number of sensors (\textbf{SM} Remark \ref{rem:size}). The structure of VIDON is illustrated and summarized in Figure \ref{fig:2}. Moreover, by construction, the new branch net $\branch$ in VIDON \eqref{eq:tb-deeponet} is a permutation-invariant function that takes finite sets of variable sizes as input. 

\begin{figure}[ht!]
\centering
\includegraphics[width=\linewidth]{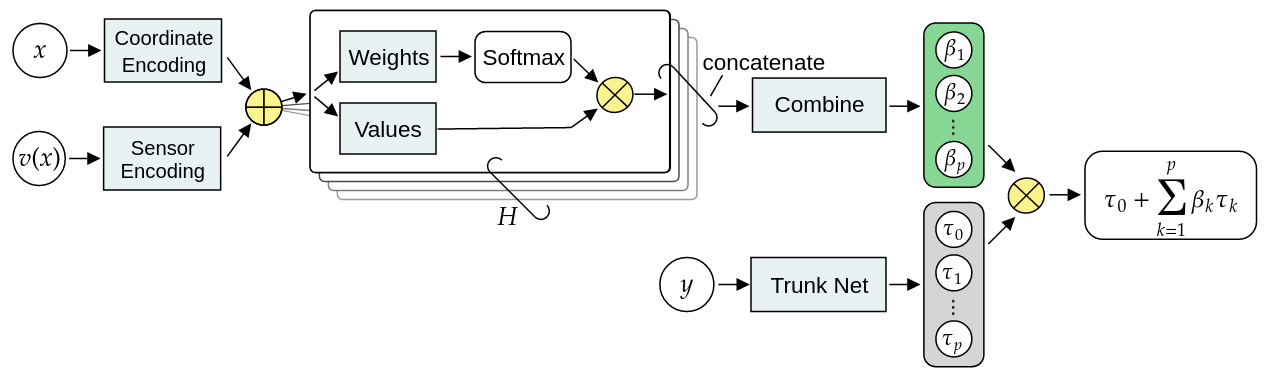}
\caption{Structure of VIDON \eqref{eq:tb-deeponet}. The  coordinates and values at each sensor are encoded through MLPs and are processed through another set of MLPs to compute weights and values. The resulting convex combination of the outputs from different sensors constitutes the output of a single head. Multiple heads are concatenated and combined through another MLP to yield the branch net, which is then combined with the trunk net to obtain the VIDON output. Light blue shading represents a MLP.}
\label{fig:2}
\end{figure}

\paragraph{Related work.}
Our proposed architecture, VIDON \eqref{eq:tb-deeponet}, is related to the DeepONet operator learning framework of \cite{deeponets} (see \cite{ChenChen1995} for the original (shallow) version of operator networks). As in the case of DeepONets, the underlying operator is approximated by VIDON with branch and trunk nets. While the structure of trunk nets is identical in both architectures, the 
branch net of VIDON is different as it allows variable number of sensors at variable locations. In particular, the branch net of VIDON is permutation-invariant. 
Moreover, one can recover DeepONets \eqref{eq:deeponet} as a special case of VIDON \eqref{eq:tb-deeponet} by setting $H=m$, $\Psi_c \equiv 0$, $\Psi_v \equiv \Id$, $\Tilde{\omega}^{(\ell)}_k = \delta_{k\ell}$ and $\Tilde{\nu}^{(\ell)} \equiv \Id $. 
Just like DeepONets, VIDONs can also be viewed as a neural operator by slightly modifying the construction proposed in \cite[Section 3.2]{NO}. Moreover, given the fact that FNOs can viewed as DeepONets with a fixed Fourier basis for trunk nets and sensors located on Cartesian grids \cite[Theorem 36]{kovachki2021universal}, fixing the Fourier basis for trunk nets of VIDON recovers FNO. However, FNO is only limited to sensors located at Cartesian grid points. Compared to both DeepONets and FNOs, VIDON's main advantage lies in its ability to handle input functions with both a variable number of sensor points as well as different sensor locations for each sample.    

Given it's designed to be invariant with respect to permutations of sensor locations, VIDON is related to deep learning architectures that approximate functions on sets. These include \emph{deep sets} \cite{zaheer2017deep,wagstaff2019limitations} as well as Neural Processes \cite{garnelo2018neural, garnelo2018conditional}, and Attentive Neural Processes \cite{kim2018attentive}. However, the branch net in VIDON has a more general structure, based on weighted sums of inputs \eqref{eq:tb-deeponet}. 

Finally, there are connections between VIDON and the extensive literature on transformers. To start with, the idea of positional encoding \eqref{eq:psi} is borrowed from the field of sequence modelling. Instead of the popular choice of fixing the encoding (e.g. a sinusoidal encoding in the transformers of \cite{vaswani2017attention}), we learn the optimal encoding using a MLP in VIDON. Another relevant transformer-based architecture is the \emph{Set Transformer} \cite{lee2019set}, which uses permutation-invariant attention blocks to learn functions on sets. However, the main difference between transformer based architectures and VIDON lies in the fact that the \emph{attention} head in transformers accounts for correlations between the different coordinates, whereas the \emph{head} \eqref{eq:nu} in VIDON does not require any interactions between different coordinates and only assigns a weight for each coordinate based on its intrinsic value. As a result, VIDON scales as $\bigO(m)$ instead of $\bigO(m^2)$ (as standard transformers do), with $m$ being the number of sensors. It is also worth mentioning recent works that use transformer-type architectures such as the Fourier and Galerkin transformers of \cite{cao2021choose} as well as the coupled attention-based LOCA framework of \cite{loca}. Although these architectures are shown to perform well on numerical experiments, they do not possess the rigorous theoretical guarantees nor the flexibility of inputs of VIDON.


\section{Rigorous analysis of Variable-Input Deep Operator Networks}\label{sec:analysis}
\paragraph{Universal approximation theorem.} Conventional neural networks are universal in the sense that they can approximate any continuous function. In similar vein, one can show that operator learning frameworks such as DeepONets (in \cite{LMK1}) and FNOs (in \cite{kovachki2021universal}) are also \emph{universal} in being able to, in principle, approximate any continuous operator. As a first step in the rigorous analysis of VIDON, we prove ({\bf SM} \ref{proof-universal}) the following universal approximation theorem for VIDON, 
\begin{theorem}\label{thm:universal-approximation}
Let $\cG:\cX\to H^s(U)$ be an $\alpha$-Hölder continuous operator, let $\mu$ be a measure on $L^2(D)$ whose covariance operator has a bounded eigenbasis with eigenvectors $\{\lambda_j\}_{j\in\N}$ and let $\cE:\cX \to (\R^{d+d_v})^{\leq M}$ for $M\in\N$ be a random encoder i.e., \eqref{eq:encoder} with the underlying sensor points being randomly drawn from the uniform distribution on $D$. Then for every $p\in \N$, there exists a VIDON $\cN: \R^{\leq M}\to H^s(U)$ \eqref{eq:tb-deeponet}, with $p$ branch and trunk nets such that for every $m\in\N$ with $m\leq M$ it holds with probability 1 that,
\begin{equation}\label{eq:universal-approximation}
    \norm{\cG(u_0)-\cN(\cE(u_0))}_{L^2(\mu)} \lesssim \left(\sum\nolimits_{j>m/C\log(m)}\lambda_j\right)^{\nicefrac{\alpha}{2}} + p^{-s/d}, 
\end{equation}
where the constant $C>0$ only depends on $\cG$ and $\mu$. 
\end{theorem}
Thus, any H\"older-continuous operator that maps into a Sobolev space can be approximated by a VIDON to desired accuracy. These assumptions are satisfied by a wide variety of operators arising in ODEs and PDEs as seen in the following. To illustrate the quantitative error bound in \eqref{eq:universal-approximation}, we readily apply \eqref{eq:universal-approximation} to the often encountered example of Gaussian measures (see for instance \cite[Section 3.5.1]{LMK1} for definitions) to obtain, 
\begin{corollary}\label{cor:universal-approximation}
If the measure $\mu$ in the statement of Theorem \ref{thm:universal-approximation} is Gaussian, then \eqref{eq:universal-approximation} reduces to
\begin{equation}
     \norm{\cG(u_0)-\cN(\cE(u_0))}_{L^2(\mu)} \lesssim \exp(-Cm^2/\log(m)^2) + p^{-s/d}. 
\end{equation}
\end{corollary}
Hence for Gaussian measures, we see that the error decays exponentially in the number of sensors. 

\paragraph{Computational complexity of VIDON.} The universal approximation theorem (Theorem \ref{thm:universal-approximation}) shows that one can find a VIDON such that the approximation error for the underlying operator can be made as small as needed, as long as one chooses $m$ and $p$ to be large enough. Although this result provides information about the required number of sensors $m$ and number of branch and trunk nets $p$ to obtain a certain accuracy, it does not reveal how large the underlying networks should be. 

Given the underlying infinite-dimensional nature of the operator, one observes that the network size of both DeepONets \cite[Remark 3.2]{LMK1} and FNOs \cite[Remark 22]{kovachki2021universal} can grow (super)exponentially with respect to the desired accuracy, while approximating general continuous operators. By leveraging the relation between VIDON and DeepONets, one can translate these results to show that the needed size of VIDON (see {\bf SM} Remark \ref{rem:size} for definitions) to approximate certain operators $\cG$ to an accuracy $\epsilon>0$ may grow as $\epsilon^{-\xi(\epsilon)}$, where $\xi:(0,\infty)\to\N$ is a monotonically decreasing function (see {\bf SM} Remark \ref{rem:yarotsky} for details). Such exponential growth clearly inhibits efficient approximation of operators. However, in \cite{LMK1} and \cite{kovachki2021universal}, it was shown that in the special case of operators arising in a wide variety of PDEs, DeepONets and FNOs, respectively, can mitigate this exponential growth in size as one can prove that the corresponding network size will only grow \emph{polynomially} with respect to the error. We will prove similar \emph{efficient approximation} results for VIDON. To do so, we follow \cite{LMK1,kovachki2021universal} and consider several prototypical examples of operators arising in PDEs, defined on the $d$-dimensional torus $D = \T^d=[0,2\pi)^d$. This should however not be seen as a restriction, as for every  Lipschitz domain $D$ with $\overline{D}\subset \T^d$ there exists a (continuous and linear) periodic extension operator $\mathfrak{E}:W^{m,p}(D)\to W^{m,p}(\T^d)$, $m\in\N, 1\leq p\leq \infty$, such that $\mathfrak{E}(u)\vert_D = u$ and $\mathfrak{E}(u)$ and its derivatives are $\T^d$-periodic \cite[Lemma 41]{kovachki2021universal}. 

\paragraph{Darcy-type elliptic PDE.} 
We start with a standard elliptic PDE that models the steady-state pressure for a fluid flowing in a porous medium according to the Darcy’s law, or diffusion of heat in a material with variable thermal conductivity. 
Given $ \cA^\ell_\lambda(\T^d) = \{a\in H^\ell(\T^d)\:\vert\: \norm{a}_{H^\ell(\T^d)}\leq \lambda^{-1}, \: \norm{a-1}_{L^\infty(\T^d)}\leq 1-\lambda\}$ for $\lambda>0$ and $\ell\in\N$, we consider the operator $\G: \cA^\ell_\lambda(\T^d) \to \dot{H}^1(\T^d)$, $a \mapsto u$, where $u$ solves the Darcy equation 
\begin{equation}
\label{eq:darcy}
-\nabla \cdot \left(a\nabla u\right) = f,
\quad
\fint_{\T^d} u(x) \, dx = 0,
\end{equation}
on the periodic torus $\T^d$, with right-hand side $f \in \dot{H}^{k-1}$. For this nonlinear operator that maps the coefficients of a PDE to its solution, we have the following result on its approximation by VIDON,

\begin{theorem}\label{thm:darcy}
Let $\ell>d$, $\lambda\in (0,1)$ and let $\cE$ be a random encoder. There exists a VIDON $\cN:(\R^{d+1})^\cF \to H^1(\T^d)$ \eqref{eq:tb-deeponet}, with $p$ trunk nets such that for every $a\in \cA^\ell_\lambda(\T^d)$ and $m\geq p^{1+\ell/d}$,
\begin{equation}\label{eq:darcy-estimate}
    \E{\norm{\cN(\cE_m(a))-\cG(a)}_{L^2(\T^d)}} \leq  C(p^{1-\ell/d}+p^{2}m^{-1/2}). 
\end{equation}
It holds that $\mathrm{depth}(\branch) = \bigO(\log(p))$, $\mathrm{width}(\branch) = \bigO(p^{(d+1)/d})$, $\mathrm{size}(\branch) = \bigO(p^3)$,  $\mathrm{depth}(\trunk) = 2$, $\mathrm{width}(\trunk) = \bigO(p^{(d+1)/d})$ and $\mathrm{size}(\trunk) = \bigO(p^{(d+2)/d})$ In particular, to obtain an accuracy of $\epsilon>0$ in \eqref{eq:darcy-estimate}, it suffices that $ \mathrm{size}(\cN) = \bigO(\epsilon^{-3d/(\ell-d)})$ and $m=\bigO(\epsilon^{-2(\ell+d)/(\ell-d)})$.
\end{theorem}
We sketch the proof here, while deferring the details to {\bf SM} \ref{proof-darcy}. First, we approximate the coefficient $a$ by a truncated version of its Fourier series and then further approximate the underlying Fourier coefficients based on the random encoder $\cE_m(a)$ using a Monte-Carlo approach ({\bf SM} Lemma \ref{lem:MC-Fourier}). Emulating a pseudospectral numerical scheme for the Darcy equation by neural networks as in \cite[Theorem 3.5]{kovachki2021universal} and approximating the Fourier basis by neural networks ({\bf SM} Lemma \ref{lem:rec-fourier}) then gives rise to a suitable approximation of the operator $\cG(a)$. We observe from the above theorem that, given an error tolerance $\epsilon$, the size of VIDON and the number of sensors only grow polynomially in $\epsilon$. 

\paragraph{Nonlinear parabolic PDE: Allen-Cahn equation.}
Next, we consider a nonlinear time-dependent \emph{parabolic PDE}, the so-called Allen-Cahn equation, which models reaction-diffusion phenomena with phase separations and transitions,
\begin{equation}
\label{eq:AC}
    \partial_t u = \Delta_x u + \frac{u(u^2-1)}{\epsilon^2}, \qquad u(t=0)=u_0
\end{equation}
where $u$ is the state and $\epsilon>0$ is a parameter that controls the contribution of the nonlinear term. We are interested in approximating the solution or time-evolution operator $\cG: C^\ell(\T^d) \to C^\ell(\T^d): u_0\mapsto u(T)$, mapping initial conditions $u_0$ to solutions at final time $T$. We have the following theorem (proved in {\bf SM} \ref{proof-AC}) on the \emph{efficient} approximation of this operator with VIDON,

\begin{theorem}\label{thm:allen-cahn}
Let $p, \ell\in\N$ with $\ell\geq 4$ and let $\cE$ be a random encoder. There exists a VIDON $\cN:(\R^{d+1})^\cF\to C(\T^d)$ \eqref{eq:tb-deeponet}, with $p$ trunk nets such that for every $u_0\in C^\ell(\T^d)$ and $m\in\N$, 
\begin{equation}\label{eq:AC-estimate}
\E{\Vert \G(u_0) - \cN(\cE_m(u_0)) \Vert_{L^2(\T^d)}}
\le
C(p^{-\ell/d}+p^{3(\ell+d)/(\ell-d)}m^{-1/2}).
\end{equation}
It holds that $\mathrm{depth}(\branch) = \bigO(p^{1+\ell/d})$, $\mathrm{width}(\branch) = \bigO(p^{(d+\ell)/2+1/d})$, $\mathrm{size}(\branch) = \bigO(p^{(d+\ell)/2+2/d})$, $\mathrm{depth}(\trunk) = 2$, $\mathrm{width}(\trunk) = \bigO(p^{(d+1)/d})$ and $\mathrm{size}(\trunk) = \bigO(p^{(d+2)/d})$. In particular, to obtain an accuracy of $\epsilon>0$ in \eqref{eq:AC-estimate}, it suffices that $\mathrm{size}(\cN) = \bigO(\epsilon^{-d(d+\ell)/2\ell - 2\ell})$ and $m=\bigO(\epsilon^{-2-6d(\ell+d)/\ell(\ell-d)})$. 
\end{theorem}

\paragraph{Navier-Stokes equations.} The motion of a viscous, incompressible Newtonian fluid is modeled by the well-known incompressible Navier-Stokes equations,
\begin{align} \label{eq:NS}
\begin{aligned}
\partial_t u + u\cdot \nabla u + \nabla p = \nu \Delta u, \quad
\div(u) = 0, \; u(t=0) = u_0.
\end{aligned}
\end{align}
Here, $u \in \R^d$ is the velocity and $p\in \R$ is the pressure of the fluid, $\nu$ the viscosity and the initial velocity is denoted by $u_0$. For simplicity, we assume periodic boundary conditions in the domain $\T^d$ and consider solutions of \eqref{eq:NS} that are in the set $\cV \subset C([0,T]; H^r) \cap C^1([0,T];H^{r-2})$ (see {\bf SM} Definition \ref{def:V-NS} for notation). Writing $\cV_t := \{u(t)\:\vert\:u\in \cV\}$, we want to approximate the solution operator $\cG$ of \eqref{eq:NS}, mapping initial data $u_0 = u(t=0)$, to the solution $u(T)$ at $t=T$. We have the following theorem (proved in {\bf SM} \ref{proof-NS}) on \emph{efficient} approximation of this operator with VIDON,

\begin{theorem}\label{thm:navier-stokes}
Let $p\in\N$, $r\ge d/2 +2$ and let $\cE$ be a random encoder. There exists a Variable-Input Deep Operator Network $\cN:(\R^{d+1})^\cF \to H^1(\T^d)$ with $p$ trunk nets such that for every $u_0\in \cV_0$ and $m\in\N$, 
\begin{equation}\label{eq:NS-estimate}
\E{\Vert \G(u_0) - \cN(\cE_m(u_0)) \Vert_{L^2(\T^d)}}
\le
C(p^{-r/d}+pm^{-1/2}).
\end{equation}
It holds that $\mathrm{depth}(\branch) = \bigO(\log(p))$, $\mathrm{width}(\branch) = \bigO(p^{(d+1)/d})$, $\mathrm{size}(\branch) = \bigO(p^3)$,  $\mathrm{depth}(\trunk) = 2$, $\mathrm{width}(\trunk) = \bigO(p^{(d+1)/d})$ and $\mathrm{size}(\trunk) = \bigO(p^{(d+2)/d})$. In particular, to obtain an accuracy of $\epsilon>0$ in \eqref{eq:darcy-estimate}, it suffices that $ \mathrm{size}(\cN) = \bigO(\epsilon^{-3d/r})$ and $m=\bigO(\epsilon^{-2-2d/r})$.
\end{theorem}

\section{Experiments}
\label{sec:4}

\paragraph{Setting.} We will compare the performance of VIDON to baselines, DeepONets \cite{deeponets} and FNOs \cite{FNO}. As the main point of VIDON is the flexibility with respect to the number of sensors and their locations, we will consider the following configurations (with increasing levels of difficulty) in each experiment (see {\bf SM} Figure \ref{fig:SM1} for illustrations). {\em Regular Grid}: The sensors are located on a regular Cartesian grid, with equally spaced grid sizes in each direction, on the domain $D$. This is an idealized situation (particularly for training data based on physical measurements) and rarely holds in practice. {\em Irregular Grid}: Here, the sensor locations are on a fixed non-Cartesian grid for all training and test samples. {\em Missing Data.} To model missing sensor data, we randomly delete a certain proportion (up to $20\%$) of sensor locations from the original Cartesian grid. {\em Perturbed Grid}: We model both missing data as well as uncertainty in sensor locations by randomly perturbing the original sensor locations (on a regular Cartesian Grid) and deleting up to 20\% of them. {\em Random Sensor Locations.} Here, although the number of sensors is fixed across samples, their location is chosen randomly for each sample. {\em Variable Random Locations.}  We randomly vary both the number (with a variance of up to $10\%$) as well as the location of sensors, across samples, in this configuration. 

Given this lineup of configurations, FNO can only be evaluated for the idealized configuration ({\em Regular Grid}), DeepONets can be evaluated for {\em Regular Grid} and {\em Fixed Irregular Grid} whereas VIDON is designed to handle every single configuration described above. For each of these configurations, our training set contains $1000$ samples whereas the test set has $5000$ samples. The details of the training and hyperparameters for each configuration and architecture is provided in {\bf SM} \ref{app:td}

\begin{table}[h]
\caption{Mean relative Test errors in $L^2(D)$ for the Darcy flow problem for different configurations of sensors. The symbol "-" implies that the model could not be used for this configuration. 
}
\label{tab:darcy}
\begin{center}
\begin{tabular}{llccc}
\toprule
{\em Configuration} & $\#$ (Sensors) & \textbf{FNO} &\textbf{DeepONet} & \textbf{VIDON} \\
\midrule
{\em Regular Grid} & $51\times 51$ & $0.76\%$ & $1.48\%$ & $1.29\%$ \\
{\em Irregular Grid} & $51^2=2601$ & - & $1.52\%$ & $1.48\%$ \\
{\em Missing Data} & $[2081,2601]$ &- & - & $1.77\%$ \\
{\em Perturbed Grid}  & [$2341$,$2861$] &- &- & $1.68\%$ \\
{\em Random Locations} & $2601$ &- & - & $2.58\%$ \\
{\em Variable Random Locations} &[2341,2861] & - & - & $2.55\%$ \\
\bottomrule
\end{tabular}
\end{center}
\vskip -0.1in
\end{table}

\paragraph{Darcy flow.} We consider the Darcy-type elliptic PDE \eqref{eq:darcy} on the domain $D = [0,1]^2$ with periodic boundary conditions. The corresponding operator $\cG$ maps the coefficient $a$ to the solution $u$ of \eqref{eq:darcy}. We choose the permeability coefficient $a$ for training (and test) samples from the underlying measure ${\bf N}(0,{\bf C})$, i.e. a Gaussian measure with a covariance operator ${\bf C}$, given by ${\bf C} = (-\Delta + 9 I)^{-2}$. See \cite{FNO} for details on generating such coefficients and Figure \ref{fig:1} (leftmost column) for two examples. The corresponding solutions $u$ are then computed with a centered finite-difference scheme on a $251^2$ Cartesian grid (see rightmost column in Figure \ref{fig:1} for examples) and constitute the ground truth for both training and testing. The resulting mean test error for each configuration listed above is presented in Table \ref{tab:darcy} (see {\bf SM} Table \ref{tab:darcyerbr} for the corresponding standard deviations). We observe from this table that for the idealized configuration of a regular grid, FNO outperforms both DeepONet and VIDON. The fact that FNO can yield superior performance to DeepONet on problems where both frameworks apply is well-studied \cite{NO} and even theoretically investigated \cite{kovachki2021universal}. As VIDON is related to DeepONets, we expect a similar performance of both architectures which is borne out on both the regular grid as well as the fixed irregular grid, where DeepONets and VIDON can be evaluated but FNO cannot be used. The main advantage of VIDON lies in the fact that it can be applied to the other sensor configurations where neither DeepONet nor FNO work. In all these configurations, we observe that VIDON approximates the underlying operator for Darcy flow at errors which are marginally higher than the error for the idealized configuration of sensors located on a regular grid. As expected, the error with VIDON increases with increasing difficulty of the configurations while always remaining under $2.6\%$ for this particular problem. In particular, the highest amplitude of error is realized for the \emph{random} and \emph{variable random} sensor locations. This is completely expected as in these configurations, the sensor locations are randomly drawn and with the maximum number of sensors being limited to $2861$, there could be gaps of significant area, where no sensor information is available. Nevertheless, VIDON was able to approximate the underlying operator with slightly larger error (less than a factor of $2$) over the idealized configuration of a regular grid. 

\begin{table}[h]
\caption{Mean relative Test errors in $L^2(D\times (0,T))$ for the Allen-Cahn PDE for different configurations of sensors. The symbol "-" implies that the model could not be used for this configuration. 
}

\label{tab:AC}
\begin{center}
\begin{tabular}{llcc}
\toprule
{\em Configuration} & $\#$ (Sensors) &\textbf{DeepONet} & \textbf{VIDON} \\
\midrule
{\em Regular Grid} & $26\times 26$ & $0.34\%$ & $0.26\%$ \\
{\em Irregular Grid} & $26^2=676$ & $0.34\%$ & $0.27\%$ \\
{\em Missing Data} & $[541,676]$ & - & $0.63\%$ \\
{\em Perturbed Grid}  & [$608$,$744$] &- & $0.83\%$ \\
{\em Random Locations} & $676$ & - & $1.21\%$ \\
{\em Variable Random Locations} &[$608$,$744$] & - & $1.20\%$ \\
\bottomrule
\end{tabular}
\end{center}
\vskip -0.1in
\end{table}

\paragraph{Allen-Cahn equation.} Next, we consider the Allen-Cahn equation \eqref{eq:AC} on the spatial domain $D = [0,2]^2$, with initial data $u_0 = u(x,y,0)$ that corresponds to the \emph{rotated travelling wave}, 
\begin{equation}
  u(x, y, t)
  = \frac{1}{2}
  - \frac{1}{2} \tanh \left(
    \frac{1}{2 \sqrt{2} \varepsilon}
    \begin{bmatrix} c_x \\ c_y \end{bmatrix}
    \cdot
    \begin{bmatrix} x - o_x \\ y - o_y \end{bmatrix}
    - \frac{3 t}{\sqrt{2} \varepsilon}
  \right), \quad c_x^2 + c_y^2 =1.
  \label{eq:AC_IC}
\end{equation}
The PDE is supplemented with zero-Neumann boundary conditions, in the rotated frame. We approximate the operator $\cG$ that maps the initial condition $u_0$ to the entire time-evolution $u(x,y,t)$, for all $t \in [0,T]$. The training (and test) initial conditions are generated by sampling the parameters $\varepsilon \in [0,13,0.18], o_x,o_y \in [0,2], c_x \in [0,1]$ uniformly. See {\bf SM} Figure \ref{fig:acex} for examples of the input and output of the operator $\cG$. Training and test data are extracted from the exact solution. VIDON is compared with  baseline DeepONet here. The standard form of FNO \cite{FNO} cannot be used to approximate the operator $\cG$ mapping the initial condition to the entire time-history. Although a recurrent version of FNO can be used instead \cite{FNO}, we chose not to do so as it can only be evaluated at fixed time steps and is thus not directly comparable to the non-reccurent versions of DeepONet and VIDON. The mean errors for all the configurations of sensor locations are presented in Table \ref{tab:AC} (standard deviations are presented in {\bf SM} Table \ref{tab:ACerbr}) and are completely consistent with those observed for the Darcy flow problem (Table \ref{tab:darcy}). The only difference being that the amplitude of test errors is significantly lower than that of the Darcy flow. This is explained by the fact that the underlying operator is easier to learn in this case.

\paragraph{Navier-Stokes equations.} In the final numerical experiment, we consider the incompressible Navier-Stokes equations \eqref{eq:NS}, with viscosity $\nu = 10^{-3}$ on the spatial domain $D = [0,1]^2$ with periodic boundary conditions. We choose the initial condition from the underlying measure ${\bf N}(0,{\bf C})$ with covariance operator ${\bf C} = 7^{3/2}(-\Delta + 49 I)^{-2.5}$, using the code from \cite{FNO}. The training (and test) data are generated with a spectral method in space and a third-order Runge-Kutta method in time \cite{SVbook} that approximates the fluid velocity and pressure. The resulting vorticity (curl of the velocity) can be readily computed. We seek to approximate the operator $\cG$ that maps the initial vorticity to the vorticity at a final time $T=5$, see {\bf SM} Figure \ref{fig:nsex} for examples of the input-output for this operator. The results for VIDON (and baselines) for different sensor configurations are presented in Table \ref{tab:NS} (see standard deviations in {\bf SM} Table \ref{tab:NSerbr}) and are qualitatively consistent with those obtained for the Darcy flow and Allen-Cahn equation. The main difference being that the amplitude of the error for both VIDON as well as the baselines is larger, corresponding to the fact that the underlying operator is harder to learn. Nevertheless, VIDON can learn with comparable accuracy to baselines at the idealized configuration, the operator for both missing data as well as perturbed sensor locations. The error is higher for random and variable random locations, but less than a factor of $2$ over the DeepONet baseline for the regular grid. 

\begin{table}[h]
\caption{Mean relative Test errors in $L^2(D)$ for the Navier-Stokes PDE for different configurations of sensors. The symbol "-" implies that the model could not be used for this configuration. 
}
\label{tab:NS}
\begin{center}
\begin{tabular}{llccc}
\toprule
{\em Configuration} & $\#$ (Sensors) & \textbf{FNO} &\textbf{DeepONet} & \textbf{VIDON} \\
\midrule
{\em Regular Grid} & $33 \times 33$ & $3.49\%$ & $4.20\%$ & $5.22\%$ \\
{\em Irregular Grid} & $33^2 = 1089$ & - & $4.33\%$ & $5.45\%$ \\
{\em Missing Data} & $[871,1089]$ &- & - & $5.64\%$ \\
{\em Perturbed Grid}  & [$980$, $1198$] &- &- & $5.34\%$ \\
{\em Random Locations} & $1089$ &- & - & $8.35\%$ \\
{\em Variable Random Locations} &[$980$, $1198$] & - & - & $8.28\%$ \\
\bottomrule
\end{tabular}
\end{center}
\vskip -0.1in
\end{table}

\section{Discussion}\label{sec:discussion}
State of the art architectures for operator learning, such as DeepONets \cite{deeponets} and FNOs \cite{FNO}, are significantly restricted in their applicability as they require the number and/or location of input sensors (whether on a Cartesian grid or randomly chosen) to be the same for every training and test sample. This implies that they cannot deal with realistic sources of data and calls into question whether having inputs as vectors of fixed length constitutes operator learning. Thus, it is desirable to design an operator learning framework that can deal with flexible inputs, where number of sensors and their locations vary across samples. We address this need by proposing VIDON \eqref{eq:tb-deeponet}, a novel operator learning framework that allows for such flexible inputs. Moreover, VIDON is designed to be permutation-invariant. We prove that VIDON is universal as well as efficient at approximating operators arising in PDEs. Numerical experiments show that VIDON is comparable in performance to DeepONet, when the latter can be applied. However, the main distinguishing feature of VIDON is its ability to handle inputs with variable number and location of sensors. We observe empirically that for these configurations, VIDON provides robust numerical performance and approximates the underlying operators accurately. At this juncture, we would like to point out some limitations of VIDON, the most prominent being the fact that it is a \emph{heavier} architecture than DeepONets as well as FNOs, requires more storage as well as training time. This seems to be the price to pay for its ability to handle flexible inputs. Moreover, the trunk net of VIDON, as in DeepONet, needs to be learned from data. In future work, one could fix the trunk nets (Fourier basis) or pretrain them \cite{fair} to further reduce the training time. Finally, a fixed positional encoding (as in transformers) could replace the learnable encoding, further reducing computational cost. 




\bibliographystyle{abbrv}
\bibliography{ref}

\begin{thebibliography}{10}

\bibitem{donet2}
S.~Cai, Z.~Wang, L.~Lu, T.~A. Zaki, and G.~E. Karniadakis.
\newblock Deep{M}\&{M}net: {I}nferring the electroconvection multiphysics
  fields based on operator approximation by neural networks.
\newblock {\em Journal of Computational Physics}, 436:110296, 2021.

\bibitem{cao2021choose}
S.~Cao.
\newblock Choose a transformer: Fourier or galerkin.
\newblock {\em Advances in Neural Information Processing Systems}, 34, 2021.

\bibitem{ChenChen1995}
T.~Chen and H.~Chen.
\newblock Universal approximation to nonlinear operators by neural networks
  with arbitrary activation functions and its application to dynamical systems.
\newblock {\em IEEE Transactions on Neural Networks}, 6(4):911--917, 1995.

\bibitem{deryck2021approximation}
T.~{De Ryck}, S.~Lanthaler, and S.~Mishra.
\newblock On the approximation of functions by tanh neural networks.
\newblock {\em Neural Networks}, 2021.

\bibitem{garnelo2018conditional}
M.~Garnelo, D.~Rosenbaum, C.~Maddison, T.~Ramalho, D.~Saxton, M.~Shanahan,
  Y.~W. Teh, D.~J. Rezende, and S.~M.~A. Eslami.
\newblock Conditional neural processes.
\newblock In J.~G. Dy and A.~Krause, editors, {\em Proceedings of the 35th
  International Conference on Machine Learning, {ICML} 2018,
  Stockholmsm{\"{a}}ssan, Stockholm, Sweden, July 10-15, 2018}, volume~80 of
  {\em Proceedings of Machine Learning Research}, pages 1690--1699. {PMLR},
  2018.

\bibitem{garnelo2018neural}
M.~Garnelo, J.~Schwarz, D.~Rosenbaum, F.~Viola, D.~J. Rezende, S.~M.~A. Eslami,
  and Y.~W. Teh.
\newblock Neural processes.
\newblock {\em CoRR}, abs/1807.01622, 2018.

\bibitem{SVbook}
D.~Gottlieb and S.~A. Orszag.
\newblock {\em Numerical analysis of spectral methods: theory and
  applications}, volume~26.
\newblock Siam, 1977.

\bibitem{grohs2018proof}
P.~Grohs, F.~Hornung, A.~Jentzen, and P.~Von~Wurstemberger.
\newblock A proof that artificial neural networks overcome the curse of
  dimensionality in the numerical approximation of {Black-Scholes} partial
  differential equations.
\newblock {\em arXiv preprint arXiv:1809.02362}, 2018.

\bibitem{HIG}
I.~Higgins.
\newblock Generalizing universal function approximators.
\newblock {\em Nature Machine Intelligence}, 3:192--193, 2021.

\bibitem{kim2018attentive}
H.~Kim, A.~Mnih, J.~Schwarz, M.~Garnelo, A.~Eslami, D.~Rosenbaum, O.~Vinyals,
  and Y.~W. Teh.
\newblock Attentive neural processes.
\newblock In {\em International Conference on Learning Representations}, 2019.

\bibitem{loca}
G.~Kissas, J.~Seidman, L.~Ferreira~Guilhoto, V.~M. Preciado, G.~J. Pappas, and
  P.~Perdikaris.
\newblock Learning operators with coupled attention.
\newblock {\em arXiv preprint arXiv:2202.01032}, 2022.

\bibitem{kovachki2021universal}
N.~Kovachki, S.~Lanthaler, and S.~Mishra.
\newblock On universal approximation and error bounds for {Fourier Neural
  Operators}.
\newblock {\em arXiv preprint arXiv:2107.07562}, 2021.

\bibitem{NO}
N.~Kovachki, Z.~Li, B.~Liu, K.~Azizzadensheli, K.~Bhattacharya, A.~Stuart, and
  A.~Anandkumar.
\newblock Neural operator: Learning maps between function spaces.
\newblock {\em arXiv preprint arXiv:2108.08481v3}, 2021.

\bibitem{LMK1}
S.~Lanthaler, S.~Mishra, and G.~E. Karniadakis.
\newblock Error estimates for {DeepOnets}: A deep learning framework in
  infinite dimensions.
\newblock {\em Transactions of Mathematics and Its Applications}, 6(1):tnac001,
  2022.

\bibitem{lee2019set}
J.~Lee, Y.~Lee, J.~Kim, A.~Kosiorek, S.~Choi, and Y.~W. Teh.
\newblock Set transformer: A framework for attention-based
  permutation-invariant neural networks.
\newblock In {\em International Conference on Machine Learning}, pages
  3744--3753. PMLR, 2019.

\bibitem{FNO}
Z.~Li, N.~Kovachki, K.~Azizzadenesheli, B.~Liu, K.~Bhattacharya, A.~Stuart, and
  A.~Anandkumar.
\newblock Fourier neural operator for parametric partial differential
  equations, 2020.

\bibitem{GKO}
Z.~Li, N.~B. Kovachki, K.~Azizzadenesheli, B.~Liu, K.~Bhattacharya, A.~M.
  Stuart, and A.~Anandkumar.
\newblock Neural operator: Graph kernel network for partial differential
  equations.
\newblock {\em CoRR}, abs/2003.03485, 2020.

\bibitem{Mpole}
Z.~Li, N.~B. Kovachki, K.~Azizzadenesheli, B.~Liu, A.~M. Stuart,
  K.~Bhattacharya, and A.~Anandkumar.
\newblock Multipole graph neural operator for parametric partial differential
  equations.
\newblock In H.~Larochelle, M.~Ranzato, R.~Hadsell, M.~F. Balcan, and H.~Lin,
  editors, {\em Advances in Neural Information Processing Systems (NeurIPS)},
  volume~33, pages 6755--6766. Curran Associates, Inc., 2020.

\bibitem{FNO1}
Z.~Li, H.~Zheng, N.~Kovachki, D.~Jin, H.~Chen, B.~Liu, K.~Azizzadenesheli, and
  A.~Anandkumar.
\newblock Physics-informed neural operator for learning partial differential
  equations.
\newblock {\em arXiv preprint arXiv:2111.03794}, 2021.

\bibitem{donet3}
C.~Lin, Z.~Li, L.~Lu, S.~Cai, M.~Maxey, and G.~E. Karniadakis.
\newblock Operator learning for predicting multiscale bubble growth dynamics.
\newblock {\em The Journal of Chemical Physics}, 154(10):104118, 2021.

\bibitem{deeponets}
L.~Lu, P.~Jin, and G.~E. Karniadakis.
\newblock {DeepONet}: Learning nonlinear operators for identifying differential
  equations based on the universal approximation theorem of operators.
\newblock {\em arXiv preprint arXiv:1910.03193}, 2019.

\bibitem{fair}
L.~Lu, X.~Meng, S.~Cai, Z.~Mao, G.~Goswami, Z.~Zhang, and G.~e. Karniadakis.
\newblock A comprehensive and fair comparison of two neural operators (with
  practical extensions) based on fair data.
\newblock {\em arXiv preprint arXiv:2111.05512v1}, 2021.

\bibitem{donet1}
Z.~Mao, L.~Lu, O.~Marxen, T.~A. Zaki, and G.~E. Karniadakis.
\newblock Deepm\&mnet for hypersonics: Predicting the coupled flow and
  finite-rate chemistry behind a normal shock using neural-network
  approximation of operators.
\newblock {\em Journal of Computational Physics}, 447:110698, 2021.

\bibitem{FNO2}
J.~Pathak, S.~Subramanian, P.~Harrington, S.~Raja, A.~Chattopadhyay,
  M.~Mardani, T.~Kurth, D.~Hall, Z.~Li, K.~Azizzadenesheli, p.~Hassanzadeh,
  K.~Kashinath, and A.~Anandkumar.
\newblock Fourcastnet: A global data-driven high-resolution weather model using
  adaptive fourier neural operators.
\newblock {\em arXiv preprint arXiv:2202.11214}, 2022.

\bibitem{tang2016implicit}
T.~Tang and J.~Yang.
\newblock Implicit-explicit scheme for the allen-cahn equation preserves the
  maximum principle.
\newblock {\em Journal of Computational Mathematics}, pages 451--461, 2016.

\bibitem{vaswani2017attention}
A.~Vaswani, N.~Shazeer, N.~Parmar, J.~Uszkoreit, L.~Jones, A.~N. Gomez,
  L.~Kaiser, and I.~Polosukhin.
\newblock Attention is all you need.
\newblock {\em Advances in neural information processing systems}, 30, 2017.

\bibitem{wagstaff2019limitations}
E.~Wagstaff, F.~Fuchs, M.~Engelcke, I.~Posner, and M.~A. Osborne.
\newblock On the limitations of representing functions on sets.
\newblock In {\em International Conference on Machine Learning}, pages
  6487--6494. PMLR, 2019.

\bibitem{yarotsky2017error}
D.~Yarotsky.
\newblock Error bounds for approximations with deep relu networks.
\newblock {\em Neural Networks}, 94:103--114, 2017.

\bibitem{zaheer2017deep}
M.~Zaheer, S.~Kottur, S.~Ravanbakhsh, B.~Poczos, R.~R. Salakhutdinov, and A.~J.
  Smola.
\newblock Deep sets.
\newblock {\em Advances in neural information processing systems}, 30, 2017.

\end{thebibliography}

\newpage 

\appendix

\section{Notation and preliminaries}\label{app:preliminaries}

We introduce notation and preliminary results regarding Sobolev spaces and the Fourier basis, as well as some auxiliary lemmas.

\subsection{Glossary of used notation}\label{sec:glossary}

\begin{table}[h!]
  \caption{Glossary of used notation.}
  \label{tab:glossary}
  \centering

\begin{tabular}{l p{0.7\textwidth}  r}
\toprule
\textbf{Symbol} & \textbf{Description} & \textbf{Section}
\\
\midrule
$d$ & spatial dimension of domain &  \\
$D,U$ & general $d$-dimensional spatial domain & \ref{sec:2} \\
$\cX$ &  input function space of the operator $\cG$; $\cX\subset L^2(D;\R^{d_v})$ & \ref{sec:2} \\
$\cY$ &  output function space of the operator $\cG$; $\cY\subset L^2(U;\R^{d_u})$ & \ref{sec:2} \\
$\cG$ &  operator of interest, $\cG:\cX\to \cY$ & \ref{sec:2} \\
$\mathfrak{X}^{\leq M}$ & the set of subsets of a set $\mathfrak{X}$ containing at most $M\in\N$ elements & \ref{sec:2} \\
$\mathfrak{X}^{\cF}$ & the set of all finite subsets of a set $\mathfrak{X}$ & \ref{sec:2} \\
$\cE$ & encoder that maps a function $f:D\to\mathfrak{X}$ to $\mathfrak{X}^{\cF}$, cf. \eqref{eq:encoder} & \ref{sec:2} \\
$\cE_m$ & encoder that maps a function $f:D\to\mathfrak{X}$ to $\mathfrak{X}^{m}$, cf. \eqref{eq:encoder}& \ref{sec:2} \\
$\cN$ & Variable-Input Deep Operator Network (VIDON); $\cN:\cE(\cX)\to\cY$ \ref{sec:2} \\
$\T^d$ & periodic torus, identified with $[0,2\pi)^d$ & \ref{sec:analysis} \\
$H^s$ & Sobolev space of smoothness $s$, with norm $\Vert \slot \Vert_{H^s}$ & \ref{sec:sobolev} \\
$\dot{H}^s$ & Sobolev space with zero mean, with norm $\Vert \slot \Vert_{\dot{H}^s}$ & \ref{sec:sobolev} \\
$W^{k,p}$ & Sobolev space of smoothness $k$, with norm $\Vert \slot \Vert_{W^{k,p}}$ & \ref{sec:sobolev} \\
$C^k$ & space of $k$ times continuously differentiable functions & 
\ref{sec:sobolev} \\
$L^2_N$ & $L^2_N\subset L^2$ trigonometric polynomials of degree $\le N$ & \ref{app:Fourier} \\
\bottomrule
\end{tabular}

\end{table}

\subsection{Sobolev spaces}\label{sec:sobolev}

Let $d\in\mathbb{N}$, $k\in\mathbb{N}_0$, $1\leq p\leq \infty$ and let $\Omega \subseteq \mathbb{R}^d$ be open. For a function $f:\Omega\to\mathbb{R}$ and a (multi-)index $\alpha \in \N^d_0$ we denote by 
\begin{equation}
    D^\alpha f= \frac{\partial^{\abs{\alpha}} f}{\partial x_1^{\alpha_1}\cdots \partial x_d^{\alpha_d}}
\end{equation}
the classical or distributional (i.e. weak) derivative of $f$. 
We denote by $L^p(\Omega)$ the usual Lebesgue space and for we define the Sobolev space $W^{k,p}(\Omega)$ as
\begin{equation}
    W^{k,p}(\Omega) = \{f \in L^p(\Omega): D^\alpha f \in L^p(\Omega) \text{ for all } \alpha\in\mathbb{N}^d_0 \text{ with } \abs{\alpha}\leq k\}. 
\end{equation}
For $p<\infty$, we define the following seminorms on $W^{k,p}(\Omega)$, 
\begin{equation}
    \abs{f}_{W^{m,p}(\Omega)} = \left(\sum_{\abs{\alpha}= m}\norm{D^\alpha f}^p_{L^p(\Omega)}\right)^{1/p} \qquad \text{for } m=0,\ldots, k, 
\end{equation}
and for $p=\infty$ we define
\begin{equation}
    \abs{f}_{W^{m,\infty}(\Omega)} =\max_{\abs{\alpha}= m} \norm{D^\alpha f}_{L^\infty(\Omega)}\qquad \qquad \text{for } m=0,\ldots, k. 
\end{equation}
Based on these seminorms, we can define the following norm for $p<\infty$,
\begin{equation}
    \norm{f}_{W^{k,p}(\Omega)} = \left(\sum_{m=0}^k \abs{f}_{W^{m,p}(\Omega)}^p\right)^{1/p}, 
\end{equation}
and for $p=\infty$ we define the norm
\begin{equation}
    \norm{f}_{W^{k,\infty}(\Omega)} =\max_{0\leq m\leq k}  \abs{f}_{W^{m,\infty}(\Omega)}. 
\end{equation}
The space $W^{k,p}(\Omega)$ equipped with the norm $\norm{\cdot}_{W^{k,p}(\Omega)}$ is a Banach space. 

We denote by $C^k(\Omega)$ the space of functions that are $k$ times continuously differentiable and equip this space with the norm $\norm{f}_{C^k(\Omega)} = \norm{f}_{W^{k,\infty}(\Omega)}$.

\begin{lemma}[Continuous Sobolev embedding]\label{lem:sobolev-embedding}
Let $d,\ell\in\N$ and let $k\geq d/2+\ell$. Then there exists a constant $C>0$ such that for any $f\in H^{k}(\T^d)$ it holds that
\begin{equation}
    \norm{f}_{C^\ell(\T^d)}\leq C \norm{f}_{H^k(\T^d)}.
\end{equation}
\end{lemma}

\subsection{Notation and results for Fourier series}  
\label{app:Fourier}
Using the notation from \cite{LMK1}, we introduce the following ``standard'' real Fourier basis $\{\fb_{\kappa}\}_{\kappa\in \Z^d}$ in $d$ dimensions. For $\kappa = (\kappa_1,\dots, \kappa_d)\in \Z^d$, we let $\sigma(\kappa)$ be the sign of the first non-zero component of $\kappa$ and we define
\begin{align}\label{eq:fb}
\fb_{\kappa}
:=
C_{\kappa}
\begin{cases}
1,
& \sigma(\kappa)=0, \\
\cos(\langle\kappa,{x}\rangle),
& \sigma(\kappa)=-1, \\
\sin(\langle\kappa,{x}\rangle),
& \sigma(\kappa)=1,
\end{cases}
\end{align}
where the factor $C_k > 0$ ensures that $\fb_{\kappa}$ is properly normalized, i.e. that $\Vert \fb_{\kappa} \Vert_{L^2(\T^d)} = 1$. Next, let $\enum: \N \to \Z^d$ be a fixed enumeration of $\Z^d$, with the property that $j \mapsto |\enum(j)|_\infty$ is monotonically increasing, i.e. such that $j\le j'$ implies that $|\enum(j)|_\infty \le |\enum(j')|_\infty$. This will allow us to introduced an $\N$-indexed version of the Fourier basis, 
\begin{equation}
\fb_j({x}) := \fb_{\enum(j)}({x}), \quad (j\in \N). 
\end{equation}

For $N\in\N$, we define the following two sets
\begin{equation}
    \cJ_N = \{0, \ldots, 2N+1\}^d, \qquad \cK_N = \{-N, \ldots, N\}^d, 
\end{equation}
following notation from \cite{kovachki2021universal}. Furthermore, we denote by $L^2_N(\T^d)$ the space of trigonometric polynomials of degree at most $N$ and we define the following $L^2$-orthogonal projection on $L^2_N(\T^d)$, 
\begin{equation}\label{eq:P_N}
    P_N: L^2(\T^d)\to L^2_N(\T^d): \sum_{k\in \Z^d} c_k \fb_k \mapsto \sum_{k\in \cK_N} c_k \fb_k. 
\end{equation}
In particular, it holds that if $u\in H^s(\T^d)$ for $s\geq 0$ then there exists a constant $C=C(\s, d)>0$ such that for any $0\leq \ell \leq s$ it holds that, 
\begin{equation}\label{eq:acc-PN}
    \norm{u-P_N u}_{H^\ell(\T^d)}\leq CN^{\ell-s}\norm{u}_{H^s(\T^d)}. 
\end{equation}
Similarly, one can define $\dot{P}_N: L^2(\T^d)\to \dot{L}^2_N(\T^d):u\mapsto P_N u - \fint_{\T^d} u(x) dx$. 

Finally, we denote by $\cI_N:C(\T^d)\to L^2_N(\T^d)$ the pseudo-spectral projection onto $L^2_N(\T^d)$ i.e., $\cI_N u$ is defined as the unique trigonometric polynomial in $L^2_N(\T^d)$ such that $\cI_N u(x_j) = u(x_j)$ for all $j\in \cJ_N$. One can verify that $\cI_N u$ is of the form, 
\begin{equation}
    \cI_N u(x) = \frac{1}{\abs{\cK_N}} \sum_{k\in \cK_N} \sum_{j\in \cJ_N}  a_{k,j} u(x_j) \fb_k(x)
\end{equation}
for fixed $a_{k,j}$. For $s, k\in\N_0$ with $s>d/2$ and $s\geq k$, and $u\in C(\T^d)\cap H^s(\T^d)$ it holds that \cite{kovachki2021universal}
\begin{equation}\label{eq:acc-trig-pol}
    \norm{u-\cI_N u}_{H^k(\T^d)}\leq C(s,d)N^{-(s-k)}\norm{u}_{H^s(\T^d)}, 
\end{equation}
for a constant $C(s,d)>0$ that only depends on $s$ and $d$. 

\subsection{Auxiliary results}

\begin{lemma}\label{lem:MC1}
Let $p\in[2,\infty)$, $d,m\in\mathbb{N}$, let $(\Omega, \mathcal{F}, \mathcal{P})$ be a probability space, and let $X_i:\Omega\to\mathbb{R}^d, i\in\{1,\ldots, m\}$, be i.i.d. random variables with $\E{\norm{X_1}}<\infty$. Then it holds that
\begin{equation}
    \left(\E{\norm{\E{X_1}-\frac{1}{m}\sum_{i=1}^m X_i}^p}\right)^{1/p} \leq 2 \sqrt{\frac{p-1}{m}}\left(\E{\norm{\E{X_1}-X_1}^p}\right)^{1/p}. 
\end{equation}
\end{lemma}
\begin{proof}
This result is \cite[Corollary 2.5]{grohs2018proof}.
\end{proof}
\begin{lemma}\label{lem:MC}
Let $p\in[2,\infty)$, $q,m\in\mathbb{N}$, let $(\Omega, \mathcal{F}, \mathcal{P})$ and $(\mathcal{D}, \mathcal{A}, \mu)$ be probability spaces, and let for every $q\in\mathcal{D}$ the maps $X_i^q:\Omega\to\mathbb{R}, i\in\{1,\ldots, m\}$, be i.i.d. random variables with $\E{\abs{X^q_1}}<\infty$. Then it holds that
\begin{equation}
    \E{\left(\int_{\mathcal{D}}\abs{\E{X^q_1}-\frac{1}{m}\sum_{i=1}^m X^q_i}^p\mu(dq)\right)^{1/p}} \leq 2\sqrt{\frac{p-1}{m}} \left(\int_{\mathcal{D}}\E{\abs{\E{X^q_1}-X_1^q}^p}\mu(dq)\right)^{1/p}.
\end{equation}
\end{lemma}
\begin{proof}
The proof involves Hölder's inequality, Fubini's theorem and Lemma \ref{lem:MC1}. The calculation is as in \cite[eq. (226)]{grohs2018proof}.
\end{proof}

\begin{lemma}\label{lem:exp-to-prob}
Let $\epsilon>0$, let $(\Omega, \mathcal{F}, \mathcal{P})$ be a probability space, and let $X:\Omega\to\mathbb{R}$ be a random variable that satisfies $\E{\abs{X}}\leq \epsilon$. Then it holds that $\mathbb{P}(\abs{X}\leq \epsilon)>0$. 
\end{lemma}
\begin{proof}
This result is \cite[Proposition 3.3]{grohs2018proof}.
\end{proof}

\begin{lemma}\label{lem:rec-fourier}
Let $s, d, p\in \N$. For any $\epsilon>0$, there exists a neural network $\bm{\tr}: \R^d \to \R^p$ with 2 hidden layers of width $\bigO(p^{\frac{d+1}{d}}+ps\ln(ps\epsilon^{-1}))$ and such that 
\begin{align} \label{eq:Fourier-NN}
p^{3/2} \max_{j=1,\dots, p} \Vert \tr_j - \fb_j \Vert_{C^s([0,2\pi]^d)}
\le \epsilon,
\end{align}
where $\fb_1,\dots, \fb_p$ denote the first $p$ elements of the Fourier basis, as in Appendix \ref{app:Fourier}. For fixed $s$, the number of non-zero weights and biases grows as $\bigO(p^{\frac{d+2}{d}})$. 
\end{lemma}

\begin{proof}
We note that each element in the (real) trigonometric basis $\fb_1,\dots, \fb_p$ can be expressed in the form 
\begin{equation}
\fb_j({x}) 
=
\cos(\kappa\cdot {x}),
\quad
\text{or}
\quad
\fb_j({x}) = \sin(\kappa\cdot {x}),
\end{equation}
for $\kappa = \kappa(j)\in \Z^d$ with $|\kappa|_\infty \le N$, where $N$ is chosen as the smallest natural number such that $p \le (2N+1)^d$. We focus only focus on the first form, as the proof for the second form is entirely similar. Define $f:[0,2\pi]^d\to\R:x\mapsto 
\kappa\cdot x$ and $g:[-2\pi dN, 2\pi dN]\to\R:x\mapsto \cos(x)$. As $f([0,2\pi]^d) \subset [-2\pi dN, 2\pi dN]$, the composition $g\circ f$ is well-defined and one can see that it coincides with a trigonometric basis function $\fb_j$. 
Moreover, the linear map $f$ is a trivial neural network without hidden layers. Approximating $\fb_j$ by a neural network $\tau_j$ therefore boils down to approximating $g$ by a suitable neural network. 

From \cite[Theorem 5.1]{deryck2021approximation} it follows that the function $g$ there exists an independent constant $R>0$ such that for large enough $t\in\N$ there is a tanh neural network $\hg_t$ with two hidden layers and $\bigO(t+N)$ neurons such that
\begin{equation}
    \norm{g-\hg_t}_{C^s([-2\pi dN, 2\pi dN])} \leq 4 (8(s+1)^3R)^{s} \exp(t-s). 
\end{equation}
This can be proven from \cite[eq. (74)]{deryck2021approximation} by setting $\delta\leftarrow\frac{1}{3}$, $k\leftarrow s$, $s\leftarrow t$, $N\leftarrow 2$ and using $\norm{g}_{C^s}=1$ and Stirling's approximation to obtain
\begin{equation}
    \frac{1}{(t-s)!}\left(\frac{3}{2\cdot 2}\right)^{t-s} \leq \frac{1}{\sqrt{2\pi(t-s)}} \left(\frac{e}{t-s}\right)^{t-s} \leq \exp(s-t) \quad \text{for } t>s+e^2. 
\end{equation}
Setting $t = \bigO(\ln(\eta^{-1})+s\ln(s))$ then gives a neural network $\hg_t$ with $\norm{g-\hg_t}_{C^s}<\eta$. Next, it follows from \cite[Lemma A.7]{deryck2021approximation} that
\begin{align}
\begin{split}
    \norm{g\circ f-\hg_t\circ f}_{C^s([0,2\pi]^d)} &\leq 16(e^2s^4d^2)^s \norm{g-\hg_t}_{C^s([-2\pi dN, 2\pi dN])} \norm{f}_{C^s([0,2\pi]^d)}^s\\
    &\leq 16(e^2s^4d^2)^s \eta (2\pi dN)^s. 
\end{split}
\end{align}
From this follows that we can obtain the desired accuracy \eqref{eq:Fourier-NN} if we set $\tau_j = \hg_{t(\eta)}\circ f$ with
\begin{equation}
    \eta = \frac{\epsilon p^{-3/2}}{16(2\pi Nd^3e^2s^4)^{s}},
\end{equation}
which amounts to $t = \bigO(s\ln(sN\epsilon^{-1}))$. As a consequence, the tanh neural network $\tau_j$ has two hidden layers with $\bigO(s\ln(sN\epsilon^{-1})+N)$ neurons and therefore, by recalling that $p\sim N^d$, the combined network $\bm{\tau}$ has two hidden layers with
\begin{equation}
    \bigO(p(s\ln(sN\epsilon^{-1})+N)) = \bigO(ps\ln(ps\epsilon^{-1})+p^{\frac{d+1}{d}})
\end{equation}
neurons. For fixed $s$, the number of non-zero weights and biases grows as $\bigO(2\cdot p\cdot N^2) = \bigO(p^{\frac{d+2}{d}})$. 
\end{proof}

\begin{lemma}\label{lem:MC-Fourier}
Let $d, \ell, K,N\in \N$, $v\in H^\ell(\T^d)$, let $(\Omega, \cF, P)$ be a probability space and let $X_1, X_2, \ldots :\Omega\to D$ be iid random variables that are uniformly distributed on $\T^d$. Define $V_{K,N}:\Omega\to C^\infty(D)$ by,
\begin{equation}
    V_{K,N}(\omega)(x) = \sum_{\abs{k}_\infty\leq K} \chat_{k,N}(\omega)\fb_k(x), \quad \text{where} \quad  \chat_{k,N}(\omega) = \frac{\abs{\T^d}}{N}\sum_{n=1}^N v(X_n(\omega))\cdot \fb_k(X_n(\omega)). 
\end{equation}
Then there exists a constant $C(\ell, d)>0$ such that for all $s\in \N_0$ with $s\leq \ell$ it holds that,
\begin{equation}
    \E{\norm{v-V_{K,N}}_{H^s(\T^d)}} \leq C(\ell,d)\norm{v}_{H^\ell(\T^d)}(K^{s-\ell}+K^{s+d}N^{-1/2}).
\end{equation}
\end{lemma}

\begin{proof}
It holds that $v$ can be written as a Fourier series, $v(x) = \sum_{k\in\Z^d}c_k\fb_k(x)$. We can then define
\begin{equation}
    v_K(x) = \sum_{\abs{k}_\infty\leq K} c_k\fb_k(x)\quad \text{where}\quad c_k = \int_{\T^d} v(x)\fb_k(x)dx. 
\end{equation}
Using \eqref{eq:acc-PN}, we find that,
\begin{equation}
    \norm{v-v_K}_{H^s(\T^d)}\leq C(\ell,d)\norm{v}_{H^\ell(\T^d)}K^{-(\ell-s)}. 
\end{equation}
By Lemma \ref{lem:MC1} it holds that, 
\begin{equation}\label{eq:MC-chat}
    \left(\E{\abs{c_k-\chat_{k,N}}^2}\right)^{1/2} \leq \frac{2\norm{v}_{L^2(\T^d)}}{\sqrt{N}}. 
\end{equation}
Consequently, we find that
\begin{equation}
\begin{split}
\E{\norm{v_K-V_{K,N}}_{H^s(\T^d)}}& = \E{\left(\sum_{\alpha\in \N_0^d, \norm{\alpha}_1\leq s} \int_{\T^d} (D^\alpha v_K(x) -D^\alpha V_{K,N}(x))^2 dx\right)^{1/2}}\\
&\leq 2(ed)^{s/2}K^s \abs{T^d}^{1/2} \sum_{\abs{k}_\infty\leq K} \E{\left(( c_k-\chat_{k,N})^2 \right)^{1/2}}\\
&\leq 2(ed)^{s/2}K^s \abs{T^d}^{1/2} \sum_{\abs{k}_\infty\leq K} \left(\E{( c_k-\chat_{k,N})^2 }\right)^{1/2}\\
&\leq 4(ed)^{s/2}K^s \abs{T^d}^{1/2} (2K+1)^d \norm{v}_{L^2(\T^d)} N^{-1/2}. 
\end{split}
\end{equation}
In the first inequality, we used that the number of terms in the sum over $\alpha$ is bounded by $2(ed)^s$ \cite[Lemma 2.1]{deryck2021approximation}, the triangle inequality for the sum over $k$ and the fact that $\norm{D^\alpha \fb_k}_{C^0}\leq K^s$ for all $k$ (both in the definition of $v_K$ and $V_{K,N}$). Jensen's inequality proves the second inequality. The third inequality follows from \eqref{eq:MC-chat}.

As a result, we find that there exists a constant $C(\ell, d)>0$ such that
\begin{equation}
    \E{\norm{v-V_{K,N}}_{H^s(\T^d)}} \leq C(\ell,d)\norm{v}_{H^\ell(\T^d)}(K^{s-\ell}+K^{s+d}N^{-1/2}). 
\end{equation}
\end{proof}

\section{Proofs of Section \ref{sec:analysis}}\label{app:analysis}

\begin{remark}\label{rem:size}
We will quantify the size of the branch net using the following estimates,
\begin{equation}\label{eq:def-size}
\begin{split}
    \mathrm{depth}(\branch) &\lesssim  \mathrm{depth}(\Psi) + \mathrm{depth}(\omega) + \max_\ell\mathrm{depth}(\Tilde{\nu}^{(\ell)}) + \mathrm{depth}(\Phi),\\
    \mathrm{width}(\branch) &\lesssim  \mathrm{width}(\Psi) + \mathrm{width}(\omega) + \textstyle\sum\nolimits_{\ell=1}^H\mathrm{width}(\Tilde{\nu}^{(\ell)}) + \mathrm{width}(\Phi ),\\
    \mathrm{size}(\branch) &\lesssim  \mathrm{size}(\Psi) + \mathrm{size}(\omega) + \textstyle\sum\nolimits_{\ell=1}^H\mathrm{size}(\Tilde{\nu}^{(\ell)}) + \mathrm{size}(\Phi ), 
\end{split}
\end{equation}
where we define the size of a neural network as its number of non-zero weights and biases. 
The size of the Variable-Input Deep Operator Network can be estimated by $\mathrm{size}(\cN)\leq \mathrm{size}(\branch)+\mathrm{size}(\trunk)$. Note that these sizes reflect the total number of unique parameters, rather than the computational complexity of evaluating a Variable-Input Deep Operator Network, as the latter will also depend on the number of sensors $m$. 
\end{remark}

\subsection{Proof of Theorem \ref{thm:universal-approximation} (Universal approximation)}\label{proof-universal}

\begin{proof}
Let $x_1, \ldots x_M$ be iid random variables on $D$, let $\cE_m:\cX\to \R^{m}: u_0\mapsto \{(x_j,u_0(x_j))\}_{j=1}^m$, $1\leq m\leq M$, be encoders and let $\cD_m:\R^m\to \cX$ be the corresponding decoders as defined in \cite[(3.38)]{LMK1}. We can then define another encoder $\cE:\cX\times \{1, \ldots, M\}\to \R^{\leq M}: (u_0, m) \to \cE_m(u_0)$ and the corresponding decoder $\cD: \R^{\leq M} \to \cX: X\to D_{\abs{X}}(X)$. The continuity of $\cD$ (where we interpret the continuity of a function of sets as in \cite[Section A.2]{wagstaff2019limitations}) follows from its definition in \cite[(3.38)]{LMK1} and the continuity of the Moore-Penrose inverse on the set of matrices of fixed size with full rank. 

We will prove that there exists a Variable-Input Deep Operator Network of the form $\cN = \cR \circ \cA$, where $\cA: \R^{\leq M} \to \R^p$ is an approximation operator and $\cR: \R^p\to C(U)$ is a reconstruction operator, both of which will be exactly defined later on. We write $\cP:C(U)\to \R^p$ for the projection operator corresponding to $\cR$. Such a decomposition has been proposed and studied in \cite{LMK1}. Moreover, it holds that $\cE\circ \cD = $ Id, $\cD\circ \cE \approx $ Id, $\cP\circ \cR = $ Id and $\cR\circ \cP \approx $ Id. 
In particular, it has been proven that the DeepONet error decomposes in the following way \cite[Lemma 3.4]{LMK1}, 
\begin{equation}\label{eq:error-deeponet-1}
    \norm{\cG(u_0)-\cN(\cE(u_0))}_{L^2(\mu)} \leq C \Lip_\alpha(\cG) \Lip(\cR\circ\cP) (\Err_\cE)^\alpha + \Lip_\alpha(\cR) \Err_\cA + \Err_\cR, 
\end{equation}
where $\Err_\cE$, $\Err_\cA$, $\Err_\cR$ are the errors related to $\cE$, $\cA$, $\cR$, respectively (see \cite[Section 3.2.1]{LMK1} for exact definitions). We will now prove an upper bound for each separate term. 

First, we use \cite[Theorem 3.7]{LMK1} to conclude that $\Err_\cE$ can be bounded  in terms of the eigenvalues of the covariance operator, 
\begin{equation}\label{eq:eta-ErrE}
    \Err_\cE \leq C\left(\sum_{j>m/C\log(m)}\lambda_j\right)^{\nicefrac{\alpha}{2}} =: \eta(m) \qquad \text{(with probability 1)}.
\end{equation}

Next, we choose $\cR$ to be a Fourier-based reconstruction operator, 
\begin{equation}
    \cR:\R^p\to C(U): (\alpha_1, \ldots, \alpha_p)\mapsto \sum_{j=1}^p\alpha_j \fbhat_j, 
\end{equation}
where the $\fbhat_j$ are neural network approximations of the Fourier basis $\{\fb_j\}_j$ (Appendix \ref{app:Fourier} and Lemma \ref{lem:rec-fourier}). 
Using \cite[Theorem 3.5]{LMK1}, we find that $\Err_\cR\leq Cp^{-s/d}$ and $\Lip(\cR\circ\cP), \Lip_\alpha(\cR)\leq C$. We thus can rewrite $\eqref{eq:error-deeponet-1}$ as,
\begin{equation}\label{eq:error-deeponet-2}
    \norm{\cG(u_0)-\cN(\cE(u_0))}_{L^2(\mu)}\leq C(\eta(m) + \Err_\cA + p^{-s/d}) \qquad \text{(with probability 1)},
\end{equation}
where $\eta$ is monotonically decreasing in $m$ and converging to 0, as in \eqref{eq:eta-ErrE}. 

It remains to bound the approximation error $\Err_\cA$, which quantifies how well $\cA$ approximates $\cP\circ \cG\circ \cD: \R^{\leq M} \to \R^p$. By definition and because of the choice of random sensors, $\cP\circ \cG\circ \cD$ is a continuous permutation-invariant function. It then follows from \cite[Theorem 4.4]{wagstaff2019limitations} (which builds upon the work of \cite{zaheer2017deep}) that $\cP\circ \cG\circ \cD$ is continuously sum-decomposable via $\R^M$, meaning that there exist continuous functions $\varphi: \R\to \R^M$ and $\rho:\R^M\to\R^p$ such that for every $X\in\R^{\leq M}$ it holds that $(\cP\circ \cG\circ \cD)(X) = \rho\left(\sum_{x\in X} \varphi(x)\right)$. By the universal approximation property of neural networks, for arbitrary $\delta, \epsilon>0$ there exist networks $\widehat{\varphi}_\epsilon $ and $\widehat{\rho}_\delta$ such that
\begin{equation}
    \norm{\varphi-\widehat{\varphi}_\epsilon}_{C^0}<\epsilon\qquad \text{and}\qquad \norm{\rho-\widehat{\rho}_\delta}_{C^0}<\delta.
\end{equation}
Moreover, $\widehat{\rho}_\delta$ is Lipschitz continuous with Lipschitz constant $L_\delta$. We can then define $\cA(X) = \widehat{\rho}_\delta \left(\sum_{x\in X} \widehat{\varphi}_\epsilon(x)\right)$ for every $X\in\R^{\leq m}$. Using the triangle inequality then gives that for any $X\in\R^{\leq M}$ and $y:=\sum_{x\in X} \varphi(x)$, 
\begin{equation}
\begin{split}
    \abs{\cA(X)-(\cP\circ \cG\circ \cD)(X)} &\leq C_\delta \sum_{x\in X}  \abs{\widehat{\varphi}_\epsilon(x)-\varphi(x)} + \abs{\widehat{\rho}_\delta(y)-\rho(y)}\\
    &\leq C_\delta M \epsilon + \delta.
\end{split}
\end{equation}
Combining with \eqref{eq:error-deeponet-2} then gives us, 
\begin{equation}
    \norm{\cG(u_0)-\cN(\cE(u_0))}_{L^2(\mu)}\leq C(\eta(m) + C_\delta M\epsilon + \delta + p^{-s/d}) \qquad \text{(with probability 1)}.
\end{equation}
Setting $\epsilon = \delta/(C_\delta M)$ and $\delta = p^{-s/d}$ then gives us the error estimate from the statement. 

We conclude the proof by observing that $\cN = \cR \circ \cA$ indeed fits in the proposed architecture of Section \ref{sec:2}. For the branch net we set $d_{enc} := d_v$, $\Psi_c := 0$, $\Psi_v \approx \Id$, $\omega := 1$, $\Tilde{\nu} := \widehat{\varphi}_\epsilon$, $H:=1$ and $\Phi := \widehat{\rho}_\delta$. The trunk nets are given by $\tau_j := \fbhat_j$ for all $j$. 
\end{proof}


\begin{remark}\label{rem:yarotsky}
If we assume that the function $\rho:\R^m\to\R^p$ from the above proof of Theorem \ref{thm:universal-approximation} is Lipschitz continuous, then it can be proven that one needs a network (roughly) of size $\bigO(\epsilon^{-m})$ to approximate $\rho$ to an accuracy $\epsilon>0$, using a result in the sense of \cite{yarotsky2017error, deryck2021approximation}. Using the function $\eta:\N\to [0,\infty)$ from \eqref{eq:eta-ErrE}, we define its approximate inverse $\xi:(0,\infty)\to \N:\epsilon\mapsto \inf\{n\in\N\: \vert \: \eta(n)<\epsilon\}$, which quantifies the needed number of sensors to obtain a certain accuracy. As a result, a lower bound for the network size to approximate $\rho$ is given by $\bigO(\epsilon^{-\xi(\epsilon)})$. Depending on the chosen measure, this can grow rapidly. For the Gaussian setting of Corollary \ref{cor:universal-approximation} it holds that $\xi(\epsilon)\sim \sqrt{\log(1/\epsilon)}$. 
\end{remark}

\subsection{Proof of Theorem \ref{thm:darcy} (Darcy flow)}\label{proof-darcy}

\begin{proof}
In the following, we let $r:= \ell-d$ and we will use notation from Section \ref{app:Fourier}. For $K,N\in\N$, one can use Lemma \ref{lem:MC-Fourier} to define a random variable $a_{K,N}$ for which it holds that $\E{\norm{a-a_{K,N}}_{H^d(\T^d)}}\leq C(K^{d-\ell}+K^{2d}N^{-1/2})$. Using a Sobolev embedding theorem, we find that, 
\begin{equation}
     \norm{a-a_{K,N}}_{L^\infty(\T^d)} \leq C \norm{a-a_{K,N}}_{H^d(\T^d)} \leq C(K^{-r}+K^{2d}N^{-1/2}).
\end{equation}
From this follows that $\inf_{x\in \T^d} a_{K,N}(x) \geq \frac{\lambda}{2}>0$ if $K$ and $N$ are suitably chosen, given that $\inf_{x\in \T^d} a(x) \geq \lambda>0$. Combining this with \cite[Lemma 4.7]{LMK1} then gives for the solution $u_{K,N}$ of the Darcy flow with $a_{K,N}:=a_{K,N}(\omega)$ (for a fixed $\omega\in\Omega$) instead of $a$, that,
\begin{equation}
    \norm{u-u_{K,N}}_{L^2(\T^d)} \leq C \norm{a-a_{K,N}}_{L^\infty(\T^d)} \leq  C(K^{-r}+K^{2d}N^{-1/2}).
\end{equation}
It also holds that $\E{\norm{a_{K,N}}_{H^\ell(\T^d)}} \leq \norm{a}_{H^\ell(\T^d)}+CK^{\ell+d}N^{-1/2} \leq \lambda^{-1}$ (from Lemma \ref{lem:MC-Fourier}) and hence by letting $N^{1/2}\geq K^{\ell+d}$ we find $\E{\norm{a_{K,N}}_{H^\ell(\T^d)}}\leq \lambda^{-1}$ where $\lambda$ might have to be redefined (increased). By \cite[Theorem 3.5]{kovachki2021universal}, for any $M\in\N$ there exists a $\Psi$-FNO $\cN_\Psi$ such that
\begin{equation}
     \norm{u_{K,N}-\cN_\Psi(a_{K,N})}_{H^1(\T^d)} \leq CM^{-r},
\end{equation}
where the width of $\cN_\Psi$ grows as $\bigO(M^d)$ and the depth grows as $\bigO(\log(M))$. A $\Psi$-FNO is a discrete realizeation of an FNO \cite[Definition 11]{kovachki2021universal} and maps $\{a_{K,N}(x_j)\}_{j\in \cJ_{2M}}$ to $\{\cN_\Psi(a_{K,N})(x_j)\}_{j\in \cJ_{M}}$, then applies a linear mapping $\cB_M$ to obtain coefficients $\{\beta_m(a_{K,N})\}_{m\in \cK_{M}}$ such that $\cN_\Psi(a_{K,N})(x) = \sum_{j\in \cK_{M}} \beta_m(a_{K,N})\fb_j(x)$. It is important to note that $a_{K,N}(x_j)$ can be written as
\begin{equation}\label{eq:darcy-akn}
    a_{K,N}(\omega)(x_j) =  \frac{\abs{\T^d}}{N}\sum_{n=1}^N\sum_{\abs{k}_\infty\leq K} a(X_n(\omega))\cdot \fb_k(X_n(\omega)) \cdot \fb_k(x_j), 
\end{equation}
which is permutation-invariant with respect to $\{X_n\}_n$. 
Next, Lemma \ref{lem:rec-fourier} guarantees that there exist a tanh neural network with two hidden layers, each with $\bigO(K^d\ln(\delta^{-1})+K^{d+1})$ neurons, such that $\norm{\fb_k-\fbhat_k}_{C^0}\leq \delta$. We then define $\widehat{a_{K,N}}$ by replacing $\fb_k$ by $\fbhat_k$ and the first multiplication in \eqref{eq:darcy-akn} by a (fixed-size) neural network $\widehat{\times}$ that approximates the multiplication operator in $C^0$-norm, which can be done to arbitrary accuracy with a fixed size neural network \cite{deryck2021approximation}. More specifically, this leads to, \begin{equation}\label{eq:darcy-akn-hat}
    \widehat{a_{K,N}}(\omega)(x_j) =  \frac{\abs{\T^d}}{N}\sum_{n=1}^N\sum_{\abs{k}_\infty\leq K} a(X_n(\omega))\widehat{\times} \:\fbhat_k(X_n(\omega)) \cdot \fb_k(x_j), 
\end{equation}
Applying the $\Psi$-FNO $\cN_\Psi$ and linear mapping $\cB_M$ as before then gives rise to the coefficients $\{\beta_m(\widehat{a_{K,N}})\}_{m\in \cK_{M}}$. We then define $\cN$ as,
\begin{equation}
    \cN(\omega)(a)(x) = \sum_{\abs{k}_\infty\leq K} \beta_m(\widehat{a_{K,N}})\fbhat_k(x). 
\end{equation}
By comparing \eqref{eq:darcy-akn} and \eqref{eq:darcy-akn-hat}, using the Lipschitz continuity of $\cN_\Psi$ and the accuracy of $\widehat{\times}$ and $\fb_k$, we find for a fixed $a$ that, 
\begin{equation}
    \norm{\cN(\omega)(a)-\cN_\Psi(\omega)(a_{K,N})}_{L^2(\T^d)} \leq CK^d\delta. 
\end{equation}
By redefining $\delta\leftarrow K^{-r-d}$ and by setting the number of sensors points as $m\leftarrow N$ and the number of branch nets as $p\leftarrow K^d=M^d$, we find the following total error estimate,
\begin{equation}
    \E{\norm{\cN(a)-\cG(a)}_{L^2(\T^d)}} \leq  C(p^{-r/d}+p^{2}m^{-1/2}). 
\end{equation}

Finally, we see that $\cN$ is indeed a Variable-Input Deep Operator Network by comparing with the architecture in Section \ref{sec:2} with $d_{enc} := d+d_v$, $\Psi\approx \Id$, $H:=1$, $\omega := 1$, 
\begin{align}
    \begin{split}
        &\Tilde{\nu}_j((x, a(x))) = \vert \T^d\vert \cdot \sum_{\abs{k}_\infty\leq K} \fb_k(x_j) \cdot a(x)\:\widehat{\times}\: \fbhat_k(x) \quad \text{for every } j\in \cJ_M,
    \end{split}
\end{align}
$\Phi  := \cB_M \circ \cN_\Psi$ and $\tau_k := \fbhat_k$. The width and depth of $\Psi$ and $\omega$ are $\bigO(1)$, furthermore it holds that $\mathrm{depth}(\Tilde{\nu})=3$ ($2$ for $\fbhat_k$ and $1$ for $\widehat{\times}$), $\mathrm{width}(\Tilde{\nu})=\bigO(p\ln(p)+p^{(d+1)/d}) = \bigO(p^{(d+1)/d})$ and $\mathrm{size}(\Tilde{\nu}) = \bigO(p^{(d+2)/d})$ (the latter as a result from Lemma \ref{lem:rec-fourier}). Also, $\mathrm{depth}(\Phi ) = \bigO(\log(p))$ and $\mathrm{width}(\Phi ) = \bigO(p)$ and therefore $\mathrm{size}(\Phi ) = \bigO(p^2\log(p))$. Using \eqref{eq:def-size} we find,
\begin{equation}
    \mathrm{depth}(\branch) = \bigO(\log(p)), \quad \mathrm{width}(\branch) = \bigO(p^{(d+1)/d}) \quad \text{and} \quad \mathrm{size}(\branch) = \bigO(p^3), 
\end{equation}
where we used the upper bound that $p^{(d+2)/d}+p^2\log(p)\lesssim p^3$. 
For the trunk net we find that $\mathrm{depth}(\trunk) = 2$, $\mathrm{width}(\trunk) = \bigO(p^{(d+1)/d})$ and $\mathrm{size}(\trunk) = \bigO(p^{(d+2)/d})$ (as a result from Lemma \ref{lem:rec-fourier}). 

Moreover, if we want to obtain an accuracy of $\epsilon>0$, we need to set $p := \epsilon^{-d/r}$ and $m:=\epsilon^{-2(d+\ell)/r}$. This gives then rise to a total size \eqref{eq:def-size} of $ \mathrm{size}(\cN) = \bigO(\epsilon^{-3d/r})$. 
\end{proof}

\subsection{Proof of Theorem \ref{thm:allen-cahn} (Allen-Cahn)}\label{proof-AC}

\begin{proof}
For $K,N\in\N$, one can use Lemma \ref{lem:MC-Fourier} to define a random variable $u_0^{K,N}$, defined by, 
\begin{equation}
    u_0^{K,N}(\omega)(x) =  \sum_{\abs{k}_\infty\leq K} \chat_{k,N}(\omega) \cdot \fb_k(x) \quad \text{where} \quad  \chat_{k,N}(\omega) = \frac{\abs{\T^d}}{N}\sum_{n=1}^N u_0(X_n(\omega))\cdot \fb_k(X_n(\omega)), 
\end{equation}
for which it holds that $\E{\norm{u_0-u_0^{K,N}}_{H^s(\T^d)}}\leq C(K^{s-\ell}+K^{s+d}N^{-1/2})$ for any $0\leq s\leq \ell$. Using a Sobolev embedding theorem we find (for a fixed realization of $ u_0^{K,N}= u_0^{K,N}(\omega)$ that $\norm{u_0-u_0^{K,N}}_{C^{0}(\T^d)} \leq \norm{u_0-u_0^{K,N}}_{H^{d}(\T^d)}$. Lemma \ref{lem:rec-fourier} guarantees that there exist a tanh neural network with two hidden layers, each with $\bigO(K^d\ln(\delta^{-1})+K^{d+1})$ neurons, such that $\norm{\fb_k-\fbhat_k}_{C^0}\leq \delta$ and in \cite{deryck2021approximation} it is proven that one can approximate the multiplication operator in $C^0$-norm to accuracy $\delta$ with a fixed size neural network $\widehat{\times}$. Using these definitions, we define
\begin{equation}\label{eq:AC-U0}
    U^0_j = \frac{\abs{\T^d}}{N}\sum_{\abs{k}_\infty\leq K} \sum_{n=1}^N u_0(X_n(\omega))\widehat{\times}\: \fbhat_k(X_n(\omega)) \cdot \fb_k(x_j)\quad \text{for } j\in \cJ_M, \: M\in\N,
\end{equation}
for which it holds that, 
\begin{equation}
     \E{\max_j \abs{U^0_j-u_0(x_j)}} \leq CK^d(K^{-\ell}+K^dN^{-1/2}+\delta).
\end{equation}

Next, we will approximate all $u(T, x_j)$ for $j\in \cJ_M$, $M\in\N$. In \cite[Theorem 4.1]{tang2016implicit} a finite difference scheme was proposed that takes $U^0_j \approx u_0(x_j)$ as input and returns an approximation $U^n_j \approx u(T, x_j)$ for $n\in \N$ with $T/n$ being the time step of the scheme. Using the refinement of this result from \cite[Theorem 4.14]{LMK1} we find the error estimate
\begin{equation}\label{eq:AC-step1}
    \E{\max_j \abs{U^n_j-u(T, x_j)}} \leq \exp(CT\norm{u}_{C^{(2,4)}([0,T]\times \T^d)}) \cdot \left(n^{-1} + M^{-2} + \E{\max_j \abs{U^0_j-u_0(x_j)}}\right). 
\end{equation}
Following \cite[Theorem 4.11]{LMK1} we emulate this finite difference scheme to create a neural network approximation $\Uhat^n_j$ of $U^n_j$. Using the results on function approximation by tanh neural networks from \cite{deryck2021approximation} we find that there exists a tanh neural network $\Uhat^M$ of width $\bigO(M^d)$ and depth $\bigO(n)$ that maps $\{U^0_j\}_j$ to $\{\Uhat^n_j\}_j$ for which it holds that, 
\begin{equation}\label{eq:AC-step2}
    \E{\max_j \abs{\Uhat^n_j-u(T, x_j)}} \leq C(n^{-1} + M^{-2} + K^d(K^{-\ell}+K^dN^{-1/2}+\delta)), 
\end{equation}
where we used \eqref{eq:AC-step1}. 

Now define $v^Z$, $Z\in\N$ where $Z$ is a divisor of $M$, as the trigonometric polynomial interpolation of $u(T)$ at the points $\{x_j\}_{j\in\cJ_Z}\subset\{x_j\}_{j\in\cJ_M}$ (see Section \ref{app:Fourier} for more information). One can then define the function $\widehat{v}^Z$ by 
\begin{equation}\label{eq:AC-step3}
    \widehat{v}^Z(x) = \frac{1}{\abs{\cK_Z}} \sum_{k\in \cK_Z} \sum_{j\in \cJ_Z}  a_{k,j} \Uhat^M_j \fb_k(x), 
\end{equation}
which inspires us to define the Variable-Input Deep Operator Network by
\begin{equation}
    \cN(u_0)(x) = \frac{1}{\abs{\cK_Z}} \sum_{k\in \cK_Z} \sum_{j\in \cJ_Z}  a_{k,j} \Uhat^n_j \fbhat_k(x), 
\end{equation}
where $\{\fbhat_k\}_k$ is a tanh neural network with two hidden layers, each with width $\bigO(Z^d\ln(\epsilon^{-1})+Z^{d+1})$ and size $\bigO(Z^d\ln(\epsilon^{-1})^2+Z^{d+2})$, such that $\norm{\fb_k-\fbhat_k}_{C^0}\leq \epsilon$ (Lemma \ref{lem:rec-fourier}). 
By comparing \eqref{eq:AC-step3} with \eqref{eq:acc-trig-pol} and combining this with \eqref{eq:AC-step2} we find that, 
\begin{equation}
\begin{split}
    \E{\norm{\cG(u_0)-\cN(u_0)}_{L^2(\T^d)}}&\lesssim 
    C (Z^{-\ell}+Z^d(n^{-1} + M^{-2} + K^d(K^{-\ell}+K^dN^{-1/2}+\delta)+\epsilon)).
\end{split}
\end{equation}

The error estimate from the statement then follows by redefining $n \leftarrow Z^{d+\ell}$, $M\leftarrow Z^{(d+\ell)/2}$, $K\leftarrow Z^{(\ell+d)/(\ell-d)}$, $\delta \leftarrow K^{-\ell}$, $\epsilon\leftarrow Z^{-d-\ell}$ and by setting the number of sensors points as $m\leftarrow N$ and the number of branch nets as $p\leftarrow Z^d$. Indeed, we find that, 
\begin{equation}
    \E{\norm{\cG(u_0)-\cN(u_0)}_{L^2(\T^d)}}\leq C(p^{-\ell/d}+p^{3(\ell+d)/(\ell-d)}m^{-1/2})
\end{equation}

Finally, we see that $\cN$ is indeed a Variable-Input Deep Operator Network by comparing with the architecture in Section \ref{sec:2} with $d_{enc} := d+d_v$, $\Psi\approx \Id$, $H:=(M/Z)^d$, $\omega := 1$, 
\begin{align}
    \begin{split}
        &\Tilde{\nu}^{(h)}((x, u_0(x))) = \vert \T^d\vert \cdot \sum_{\abs{k}_\infty\leq K} \alpha_k^{(h)} \cdot u_0(x)\widehat{\times}\: \fbhat_k(x) , \qquad 1\leq h\leq H,
    \end{split}
\end{align}
where the $\alpha_k^{(h)}$ are defined in such a way that the output of $\nu$ is equal to $\{U^j_0\}_j$ as in \eqref{eq:AC-U0}. Finally we set
$\Phi  := \Uhat^M$ and $\tau_k := \fbhat_k$. The width and depth of $\Psi$ and $\omega$ are $\bigO(1)$. It also holds that $\mathrm{depth}(\Tilde{\nu})=3$ ($2$ for $\fbhat_k$ and $1$ for $\widehat{\times}$), 
\begin{equation}
\begin{split}
    \sum_{h=1}^H \mathrm{width}(\Tilde{\nu}^{(h)}) = \bigO(H\cdot Z^{d+1}) = \bigO(M^dZ) = \bigO(Z^{d(d+\ell)/2+1}) = \bigO(p^{(d+\ell)/2+1/d}),\\
    \sum_{h=1}^H \mathrm{size}(\Tilde{\nu}^{(h)}) = \bigO(H\cdot 3\cdot Z^{d+2}) = \bigO(M^dZ^2) = \bigO(Z^{d(d+\ell)/2+2}) = \bigO(p^{(d+\ell)/2+2/d}).
\end{split}
\end{equation}
Furthermore we find that $\mathrm{depth}(\Phi ) = \bigO(n) = \bigO(p^{1+\ell/d})$ and $\mathrm{width}(\Phi )=\bigO(M^d) = \bigO(p^{(d+\ell)/2})$. Using \eqref{eq:def-size} we find,
\begin{equation}
    \mathrm{depth}(\branch) = \bigO(p^{1+\ell/d}), \quad \mathrm{width}(\branch) = \bigO(p^{(d+\ell)/2+1/d})\quad \text{and} \quad \mathrm{size}(\branch) = \bigO(p^{(d+\ell)/2+2/d}).
\end{equation}
For the trunk net we find that $\mathrm{depth}(\trunk) = 2$, $\mathrm{width}(\trunk) = \bigO(p^{(d+1)/d})$ and $\mathrm{size}(\trunk) = \bigO(p^{(d+2)/d})$. 

Moreover, if we want to obtain an accuracy of $\epsilon>0$, we need to set $p := \epsilon^{-d/\ell}$ and $m:=\epsilon^{-2-6d(\ell+d)/\ell(\ell-d)}$. This gives then rise to a total size \eqref{eq:def-size} of $\mathrm{size}(\cN) = \bigO(\epsilon^{-d(d+\ell)/2\ell - 2\ell})$. 

\end{proof}

\subsection{Proof of Theorem \ref{thm:navier-stokes} (Navier-Stokes)}\label{proof-NS}

\begin{definition}\label{def:V-NS}
Let $T>0$, $\nu \ge 0$, $d\geq 2$, $r\ge d/2 +2$. We define $\cV \subset C([0,T]; H^r) \cap C^1([0,T];H^{r-2})$ as the set of solutions of the Navier-Stokes equations \eqref{eq:NS}, such that $\sup_{u\in \cV} \Vert u \Vert_{L^2}<\infty$, and 
\begin{equation}
    \sup_{u\in \cV} \left\{
\Vert u \Vert_{C_t(H^r_x)} + \Vert u \Vert_{C^1_t(H^{r-2}_x)}
\right\}
< \infty.
\end{equation}
\end{definition}

\begin{proof}[Proof of Theorem \ref{thm:navier-stokes}]
The proof is rather similar to that of Theorem \ref{thm:darcy}. To avoid too many indices, we define $a:=u_0$. For $K,N\in\N$, one can use Lemma \ref{lem:MC-Fourier} to define a random variable $a_{K,N}$ for which it holds that $\E{\norm{a-a_{K,N}}_{H^s(\T^d)}}\leq C(K^{s-r}+K^{s+d}N^{-1/2})$ for any $0\leq s\leq r$. By \cite[Theorem 3.12]{kovachki2021universal}, for any $M\in\N$ there exits a $\Psi$-FNO $\cN^*_\Psi$ such that
\begin{equation}
     \norm{u-\cN^*_\Psi(a)}_{L^2(\T^d)} \leq CM^{-r},
\end{equation}
where the width of $\cN^*_\Psi$ grows as $\bigO(M^d)$ and the depth grows as $\bigO(\log(M))$. 
In a similar way, we can find a $\Psi$-FNO $\cN_\Psi$ for which $\norm{u-\cN^*_\Psi(a_{K,N})}_{L^2(\T^d)}$ is small. This can be done by making the small adaptation in the proof in \cite{kovachki2021universal} of using $\dot{P}_M \cI_{2M} a_{K,N}$ as input for the $\Psi$-FNO instead of $\dot{P}_M \cI_{2M} a$. 
In \cite[Section F.2.5]{kovachki2021universal}
they use that
\begin{equation}\label{eq:E0old}
    \norm{(1-\dot{P}_M \cI_{2M})a}_{L^2(\T^d)} \leq M^{-r}\norm{a}_{H^r(\T^d)}. 
\end{equation}
If we combine the estimate,
\begin{equation}
    \norm{(1-\dot{P}_M \cI_{2M})a_{K,N}}_{L^2(\T^d)} \leq C M^{-r}\left(\norm{a}_{H^r(\T^d)}+\norm{a-a_{K,N}}_{H^r(\T^d)}\right),
\end{equation}
with the properties of $a_{K,N}$ (i.e. Lemma \ref{lem:MC-Fourier} with $s=0$) then we find that
\begin{equation}
\begin{split}\label{eq:E0new}
    \E{\norm{a-\dot{P}_M \cI_{2M}a_{K,N}}_{L^2(\T^d)}}
    \leq C(K^{-r}+K^{d}N^{-1/2}+M^{-r}(C+K^{r+d}N^{-1/2})).
\end{split}
\end{equation}
By replacing \eqref{eq:E0old} with \eqref{eq:E0new} in \cite[Section F.2.5]{kovachki2021universal}, we find that there exists a $\Psi$-FNO $\cN_\Psi$ such that
\begin{equation}
    \norm{u-\cN^*_\Psi(a_{K,N})}_{L^2(\T^d)} \leq C(K^{-r}+K^{d}N^{-1/2}(1+K^rM^{-r})).
\end{equation}
The proof can be finished in a similar way to the proof of Theorem \ref{thm:darcy}. In particular, we set $m\leftarrow N$, $p\leftarrow K^d$ and $M=K$. The sizes of $\branch$ and $\trunk$ are the same in terms of $p,m$ as in the proof of Theorem \ref{thm:darcy}. However, if we want to obtain an accuracy of $\epsilon>0$, we now need to set $p := \epsilon^{-d/r}$ and $m:=\epsilon^{-2-2d/r}$. This gives then rise to a total size \eqref{eq:def-size} of $\mathrm{size}(\cN) = \bigO(p^{-3d/r})$. 
\end{proof}

\section{Details for Numerical experiments in Section \ref{sec:4}}
\label{app:expdtl}

\begin{figure}[ht!]
\centering
\includegraphics[width=\linewidth]{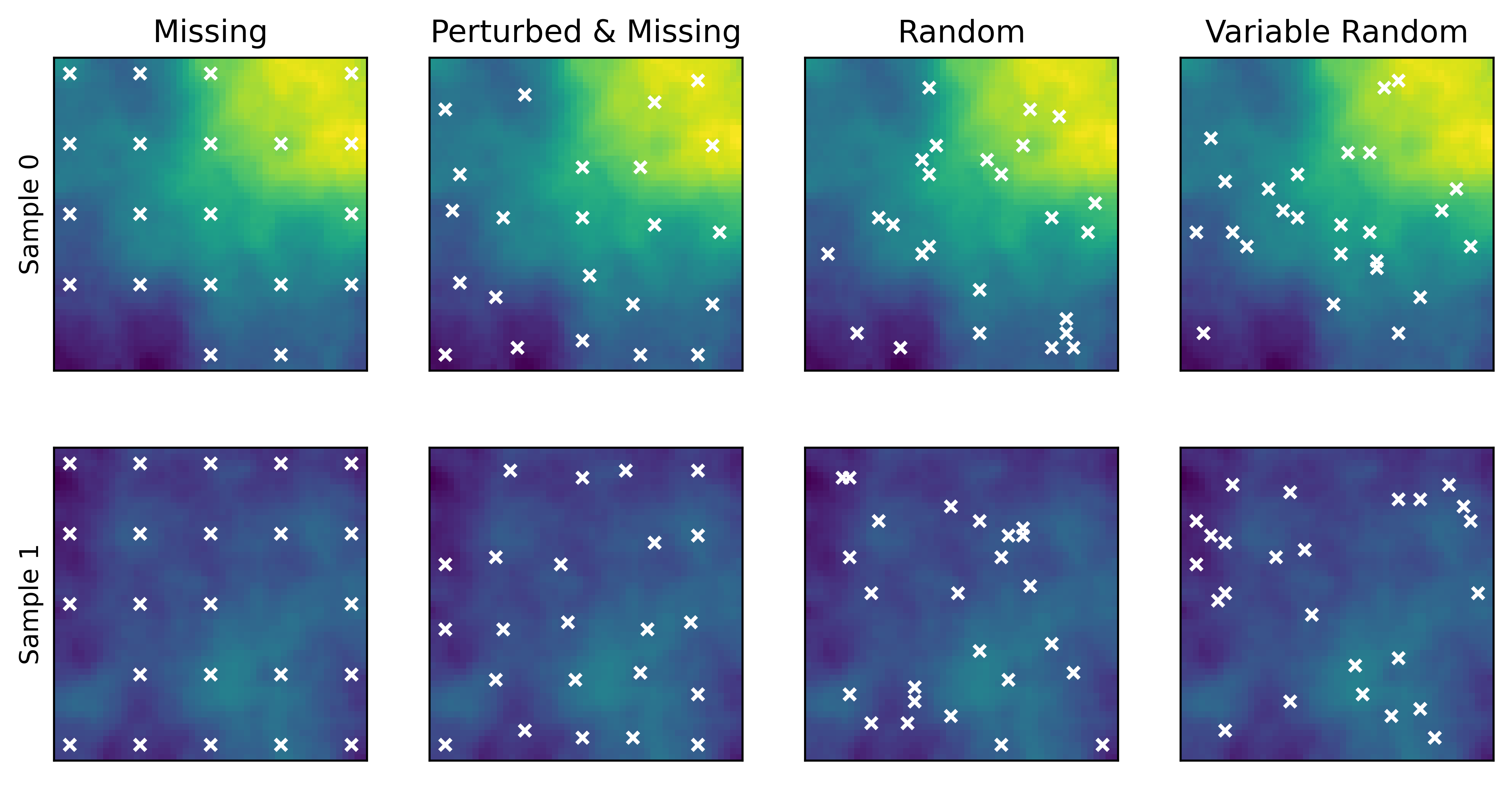}
\caption{Illustration of sensor locations for configuration where sensors are not on a grid. See Section \ref{sec:4} for the nomenclature. We consider two sample inputs for the Darcy flow test case. See columns 2 and 3 in Figure \ref{fig:1} for illustrations of regular and irregular grids respectively and plot the other sensor configurations here. }
\label{fig:SM1}
\end{figure}

\subsection{Training and Architecture Details}
\label{app:td}

For each of the three problems we create a training set containing 1000 samples, a validation set containing 32 samples and a test set containing 5000 samples.
The datasets for each coordinate configuration (regular grid, irregular grid, etc) use the same initial condition, so all models are trained on the same underlying samples.
For training and validation sets the inputs and outputs (where the model is evaluated) use the same sensor coordinates.
For example, when sensor points are drawn at random then the output is only available at the sensor locations which were also used in the input.
This is intended to simulate physical measurements where data will only ever be available at the given sensor locations.
For the test sets the input is given at the respective sensor points but the output is evaluated on a fine grid.
This is done to get an accurate estimate of how well the models learn the true solutions and to see how well the models generalize to arbitrary points in the domain. The corresponding grid sizes for (training and) testing are given in table \ref{tab:output_grid_sizes}.

During the training process, optimization is performed using the ADAM optimizer with the mean squared error loss.
In every epoch the relative $L^2$ error is monitored on the (very small) validation set.
The model that achieves the lowest relative $L^2$ error on the validation set (during the training process) is saved and selected for testing.
We obtain the model hyperparameters described below, by running grid searches over a range of hyperparameter values and selecting the hyperparameter configuration with lowest error on the validation set. Correspondingly, the relative $L^2$ test errors of these  models, with the best performing hyperparameters are reported in tables \ref{tab:darcyerbr}, \ref{tab:ACerbr} and \ref{tab:NSerbr}.


\begin{table}[h]
\caption{The left columns describe the grid resolution on which training (for a the regular grid) is performed. The right columns show the grid resolution on which inference or testing (for all coordinate configurations) is performed. Note, for the Navier-Stokes problem the input to the FNO must be available on a regular $65\times 65$ grid, this is not the case for the DeepONet and VIDON.
}
\label{tab:output_grid_sizes}
\vskip 0.15in
\begin{center}
\begin{tabular}{lcccc}
\toprule
{\em Problem} & \multicolumn{2}{c}{\em Training Grid} & \multicolumn{2}{c}{\em Test Grid} \\
 \cline{2-3} \cline{4-5}
 & space & time & space & time \\
\midrule
{\em Darcy Flow} & $51\times 51$ & - & $51\times 51$ & - \\
{\em Allen-Cahn} & $26\times 26$ & 21 & $76\times 76$ & 41 \\
{\em Navier-Stokes} & $33\times 33$ & - & $65\times 65$ & -  \\
\bottomrule
\end{tabular}
\end{center}
\vskip -0.1in
\end{table}

We normalize the input and output data of the Darcy Flow and Navier-Stokes' problems to lie in the range $[0, 1]$. The data of the Allen-Cahn problem is already in this range and thus needs no additional normalization.

Each training was run on one of the following GPUs: Nvidia GTX 1080, Nvidia GTX 1080 Ti, Nvidia Tesla V100, Nvidia RTX 2080 Ti, Nvidia Titan RTX, Nvidia Quadro RTX 6000, Nvidia Tesla A100.

In the following we describe the detailed results (including error bars) and the training parameters that were used to obtain these results. 
Exact information on the training settings can also be obtained from the settings files in the accompanying code.

\subsection{Detailed Results}

\begin{table}[h]
\caption{The mean $\pm$ standard deviation of relative test errors in $L^2(D)$ for the Darcy flow problem for different configurations of sensors. The symbol "-" implies that the model could not be used for this configuration. 
}
\label{tab:darcyerbr}
\vskip 0.15in
\begin{center}
\begin{tabular}{llccc}
\toprule
{\em Configuration} & $\#$ (Sensors) & \textbf{FNO} &\textbf{DeepONet} & \textbf{VIDON} \\
\midrule
{\em Regular Grid} & $51\times 51$ & $0.76\% \pm 0.02\%$ & $1.48\% \pm 0.01\%$ & $1.29\% \pm 0.02\%$ \\
{\em Irregular Grid} & $51^2=2601$ & - & $1.52\% \pm 0.02\%$ & $1.48\% \pm 0.07\%$ \\
{\em Missing Data} & $[2081,2601]$ &- & - & $1.77\% \pm 0.01\%$ \\
{\em Perturbed Grid}  & [$2341$,$2861$] &- &- & $1.68\% \pm 0.01\%$ \\
{\em Random Locations} & $2601$ &- & - & $2.58\% \pm 0.01\%$ \\
{\em Variable Random Locations} &[2341,2861] & - & - & $2.55\% \pm 0.01\%$ \\
\bottomrule
\end{tabular}
\end{center}
\vskip -0.1in
\end{table}

\begin{table}[h]
\caption{Mean $\pm$ standard deviation of the relative test errors in $L^2(D\times (0,T))$ for the Allen-Cahn PDE for different configurations of sensors. The symbol "-" implies that the model could not be used for this configuration. 
}
\label{tab:ACerbr}
\vskip 0.15in
\begin{center}
\begin{tabular}{llcc}
\toprule
{\em Configuration} & $\#$ (Sensors)  &\textbf{DeepONet} & \textbf{VIDON} \\
\midrule
{\em Regular Grid} & $26\times 26$ & $0.34\% \pm 0.01\%$ & $0.26\% \pm 0.02\%$ \\
{\em Irregular Grid} & $26^2=676$ & $0.34\% \pm 0.01\%$ & $0.27\% \pm 0.01\%$ \\
{\em Missing Data} & $[541,676]$ & - & $0.63\% \pm 0.02\%$ \\
{\em Perturbed Grid}  & [$608$,$744$] &- & $0.83\% \pm 0.02\%$ \\
{\em Random Locations} & $676$ & - & $1.21\% \pm 0.03\%$ \\
{\em Variable Random Locations} &[$608$,$744$]  & - & $1.20\% \pm 0.03\%$ \\
\bottomrule
\end{tabular}
\end{center}
\vskip -0.1in
\end{table}

\begin{table}[h]
\caption{Mean $\pm$ standard deviation of relative test errors in $L^2(D)$ for the Navier-Stokes PDE for different configurations of sensors. The symbol "-" implies that the model could not be used for this configuration. 
}
\label{tab:NSerbr}
\vskip 0.15in
\begin{center}
\begin{tabular}{llccc}
\toprule
{\em Configuration} & $\#$ (Sensors) & \textbf{FNO} &\textbf{DeepONet} & \textbf{VIDON} \\
\midrule
{\em Regular Grid} & $33 \times 33$ & $3.49\% \pm 0.09\%$ & $4.20\% \pm 0.02\%$ & $5.22\% \pm 0.12\%$ \\
{\em Irregular Grid} & $33^2 = 1089$ & - & $4.33\% \pm 0.04\%$ & $5.45\% \pm 0.09\%$ \\
{\em Missing Data} & $[871,1089]$ &- & - & $5.64\% \pm 0.03\%$ \\
{\em Perturbed Grid}  & [$980$, $1198$] &- &- & $5.34\% \pm 0.02\%$ \\
{\em Random Locations} & $1089$ &- & - & $8.35\% \pm 0.03\%$ \\
{\em Variable Random Locations} &[$980$, $1198$] & - & - & $8.28\% \pm 0.03\%$ \\
\bottomrule
\end{tabular}
\end{center}
\vskip -0.1in
\end{table}

\subsection{FNO Training Parameters}
We use the implementation of the FNO model provided by the authors of \cite{FNO} with some slight adjustments to make the code compatible with ours.

The FNO model for both Darcy Flow and the Navier-Stokes equations was trained using 12 modes and width 32. 
The initial learning rate was set to 1e-3 and is halved every 100 epochs.
Weight decay is set to 1e-8 and training finishes after 500 epochs. It is well-known that the training time per epoch for FNO can be significantly higher than that of DeepONet, however, FNO trains much faster i.e., with significantly fewer epochs \cite{FNO,fair}.
\subsection{DeepONet Training Parameters}

Table \ref{tab:deeponet_model_sizes} shows the model sizes used for each problem and for each of the two applicable coordinate configurations (Regular and irregular grids).
Note, we use our own implementation of the DeepONet.
We run all trainings for a maximum of 100,000 epochs.
Table \ref{tab:deeponet_training_params} shows additional training parameters.

\begin{table}[h]
\caption{Model sizes of the DeepONet used on each of the problems and for each of the two coordinate configurations (regular and irregular grids).
The notation [20, 20, 20] refers to three layers with 20 neurons each.
}
\label{tab:deeponet_model_sizes}
\vskip 0.15in
\begin{center}
\begin{tabular}{lccc}
\toprule
{\em Problem} & p & Branch Net & Trunk Net \\
\midrule
{\em Darcy Flow} & 100 & [250, 250, 250, 250] & [250, 250, 250, 250] \\
{\em Allen-Cahn} & 400 & [400, 400, 400, 400] & [500, 500, 500, 500] \\
{\em Navier-Stokes} & 100 & [250, 250, 250, 250] & [250, 250, 250, 250]  \\
\bottomrule
\end{tabular}
\end{center}
\vskip -0.1in
\end{table}

\begin{table}[h]
\caption{Training parameters of the DeepONets used on each of the the problems and for each of the two coordinate configurations (regular and irregular grids).
The third column indicates the epochs at which the learning rate is halved.
}
\label{tab:deeponet_training_params}
\vskip 0.15in
\begin{center}
\begin{tabular}{lccc}
\toprule
{\em Problem} & Initial Learning Rate & Halved at Epochs & Weight Decay \\
\midrule
{\em Darcy Flow} & 1e-4 & 5k, 30k, 60k, 90k & 1e-9 \\
{\em Allen-Cahn} & 1e-4 & 20k, 40k, 60k, 80k & 1e-9 \\
{\em Navier-Stokes} & 2e-4 & 5k, 30k, 60k, 90k & 1e-7  \\
\bottomrule
\end{tabular}
\end{center}
\vskip -0.1in
\end{table}

\subsection{VIDON Training Parameters}
The following parameters are used in all trainings for all problems and coordinate configurations.
The number of heads $H$ in VIDON \eqref{eq:tb-deeponet} was set to 4 for each configuration.
The coordinate and sensor encodings in all trainings is set to four layers with 40 neurons each. The MLPs computing weights and values in VIDON \eqref{eq:tb-deeponet} have four layers of 128 neurons each. The MLP which combines the concatenated output of each of the heads has four layers with 256 neurons each. The remaining model parameters are presented in Table \ref{tab:vidon_model_sizes} and 
additional training parameters are shown in Table \ref{tab:vidon_training_params}.
Moreover, the runtime between different sensor configurations, reported in Table \ref{tab:NSerbr} varies significantly because in the first three, the trunk net has to be evaluated at significantly fewer points in the output domain than in the remaining configurations (where a large number of different, randomly chosen points has to be evaluated). 

\begin{table}[h]
\caption{Model sizes used on each of the problems and for both of the two coordinate configurations (regular and irregular grids).
The notation [20, 20, 20] refers to three layers with 20 neurons each.
}
\label{tab:vidon_model_sizes}
\vskip 0.15in
\begin{center}
\begin{tabular}{lccc}
\toprule
{\em Problem} & $p$ & Output Neurons per Head & Trunk Net \\
\midrule
{\em Darcy Flow} & 100 & 64 & [250, 250, 250, 250] \\
{\em Allen-Cahn} & 400 & 64 & [500, 500, 500, 500] \\
{\em Navier-Stokes} & 100 & 32 & [250, 250, 250, 250]  \\
\bottomrule
\end{tabular}
\end{center}
\vskip -0.1in
\end{table}

\begin{table}[h]
\caption{Training parameters of VIDON used on each of the problems and for both of the two coordinate configurations (regular and irregular grids).
}
\label{tab:vidon_training_params}
\vskip 0.15in
\begin{center}
\begin{tabular}{lccc}
\toprule
{\em Problem} & Initial Learning Rate & Halved at Epochs & Weight Decay \\
\midrule
{\em Darcy Flow} & & &  \\
\hspace{1mm} Default & 1e-4 & 20k, 40k, 60k, 80k & 1e-9 \\
\hspace{1mm} Random / Variable Random Locations & 1e-4 & 20k, 40k, 60k, 80k & 1e-8 \\
{\em Allen-Cahn} & & &  \\
\hspace{1mm} Default & 1e-4 & 20k, 40k, 60k, 80k & 1e-9 \\
\hspace{1mm} Missing, Variable Random Locations & 2e-4 & 20k, 40k, 60k, 80k & 1e-9 \\
{\em Navier-Stokes} & & &   \\
\hspace{1mm} Default & 1e-4 & 10k, 20k, 40k, 60k, 80k & 1e-7  \\
\hspace{1mm} Random Locations & 2e-4 & 10k, 20k, 40k, 60k, 80k & 1.5e-7 \\
\hspace{1mm} Variable Random Locations & 2e-4 & 10k, 20k, 40k, 60k, 80k & 2e-7  \\
\bottomrule
\end{tabular}
\end{center}
\vskip -0.1in
\end{table}

\subsection{Additional Figures}

\begin{figure}[ht!]
\centering
\begin{subfigure}[b]{0.45\textwidth}
\includegraphics[width=\linewidth]{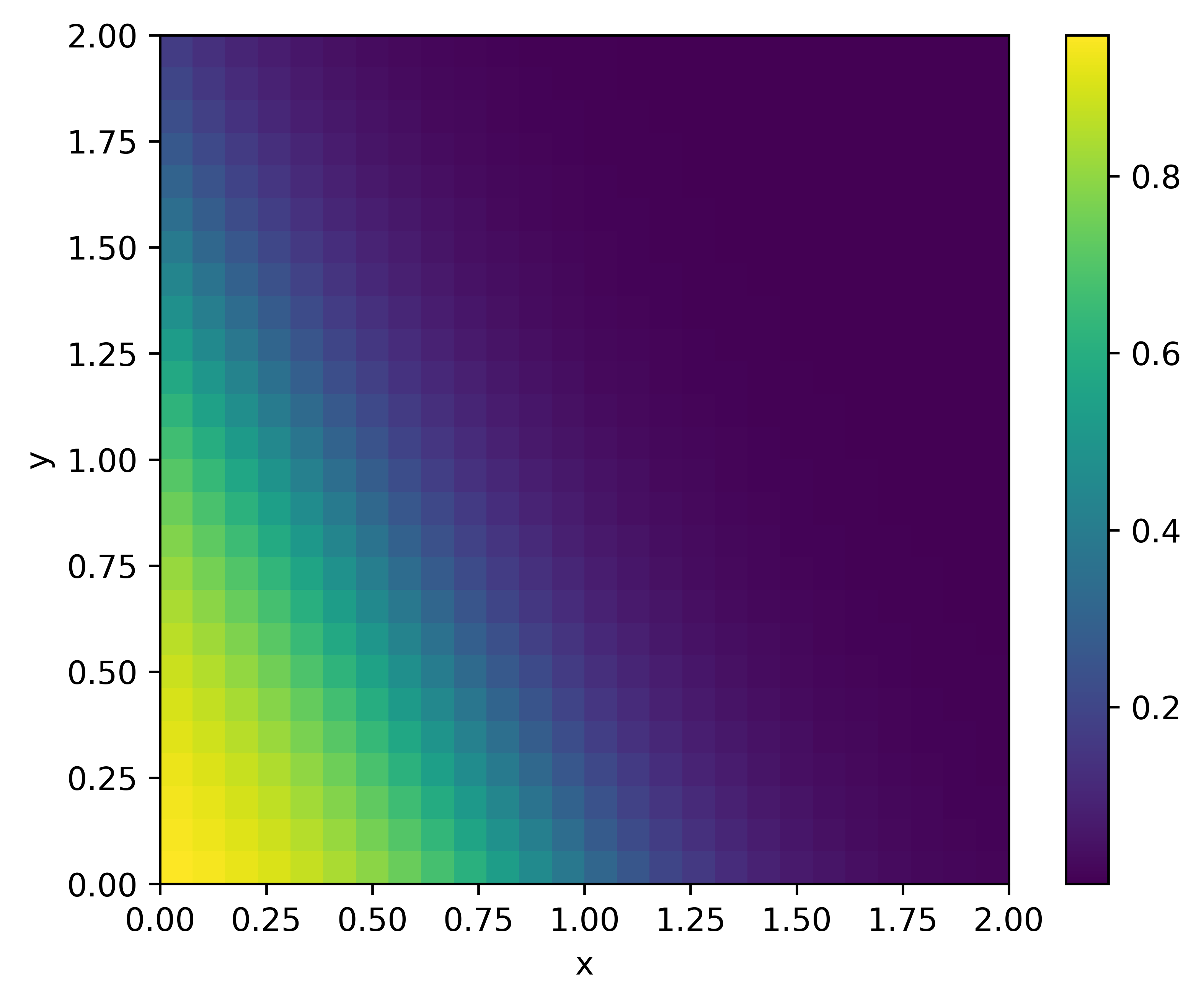}
\caption{Input}
\end{subfigure}
\begin{subfigure}[b]{0.45\textwidth}
\includegraphics[width=\linewidth]{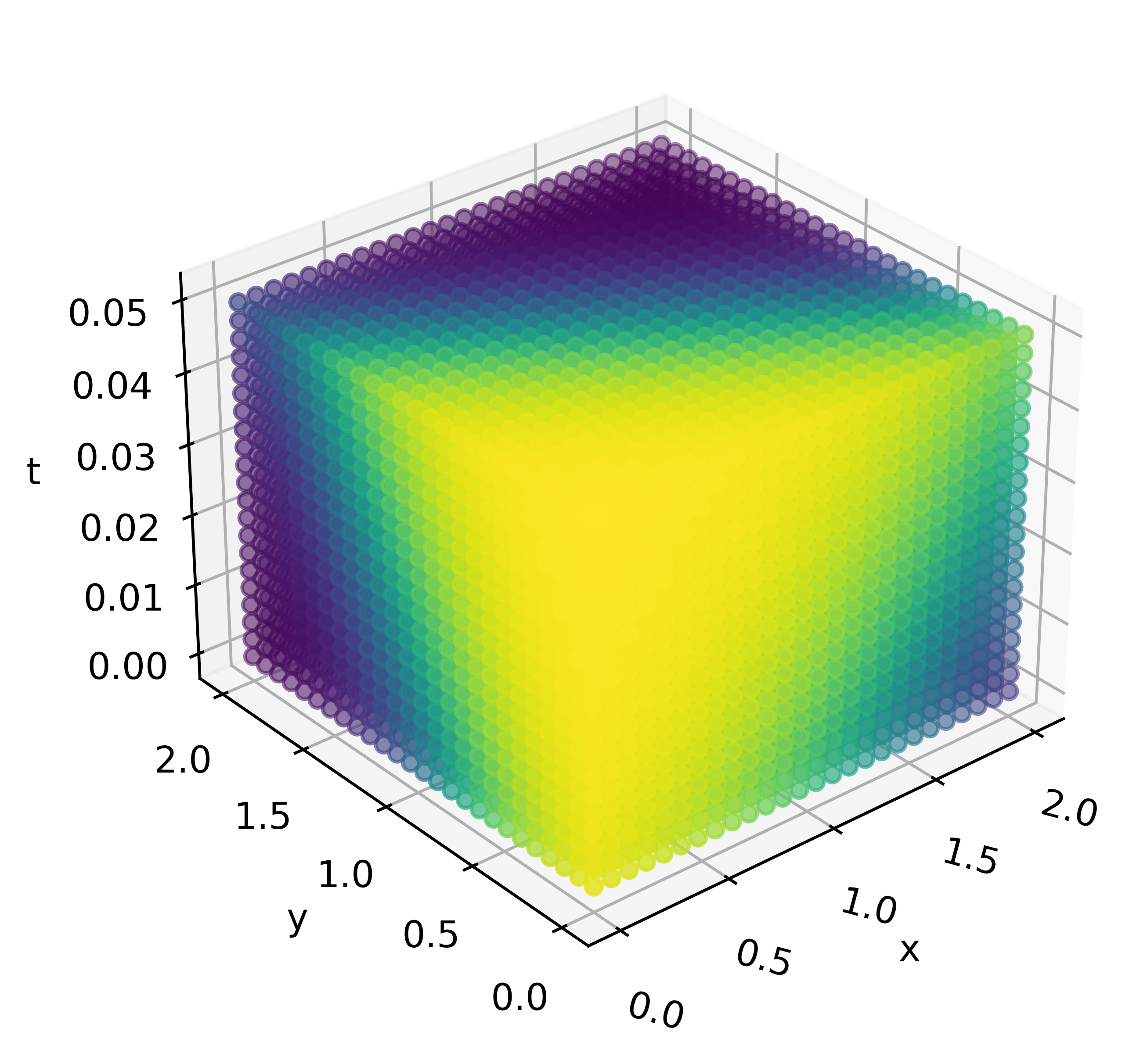}
\caption{Output}
\end{subfigure}
\caption{A sample illustrating the solution operator for the Allen-Cahn equation \eqref{eq:AC}. The input  is given by the initial conditions and the output is given by the time-history (up to time $T=0.05$) of the rotated travelling-wave solution \eqref{eq:AC_IC} of the PDE.}
\label{fig:acex}
\end{figure}

\begin{figure}[ht!]
\centering
\begin{subfigure}{0.45\textwidth}
\includegraphics[width=\linewidth]{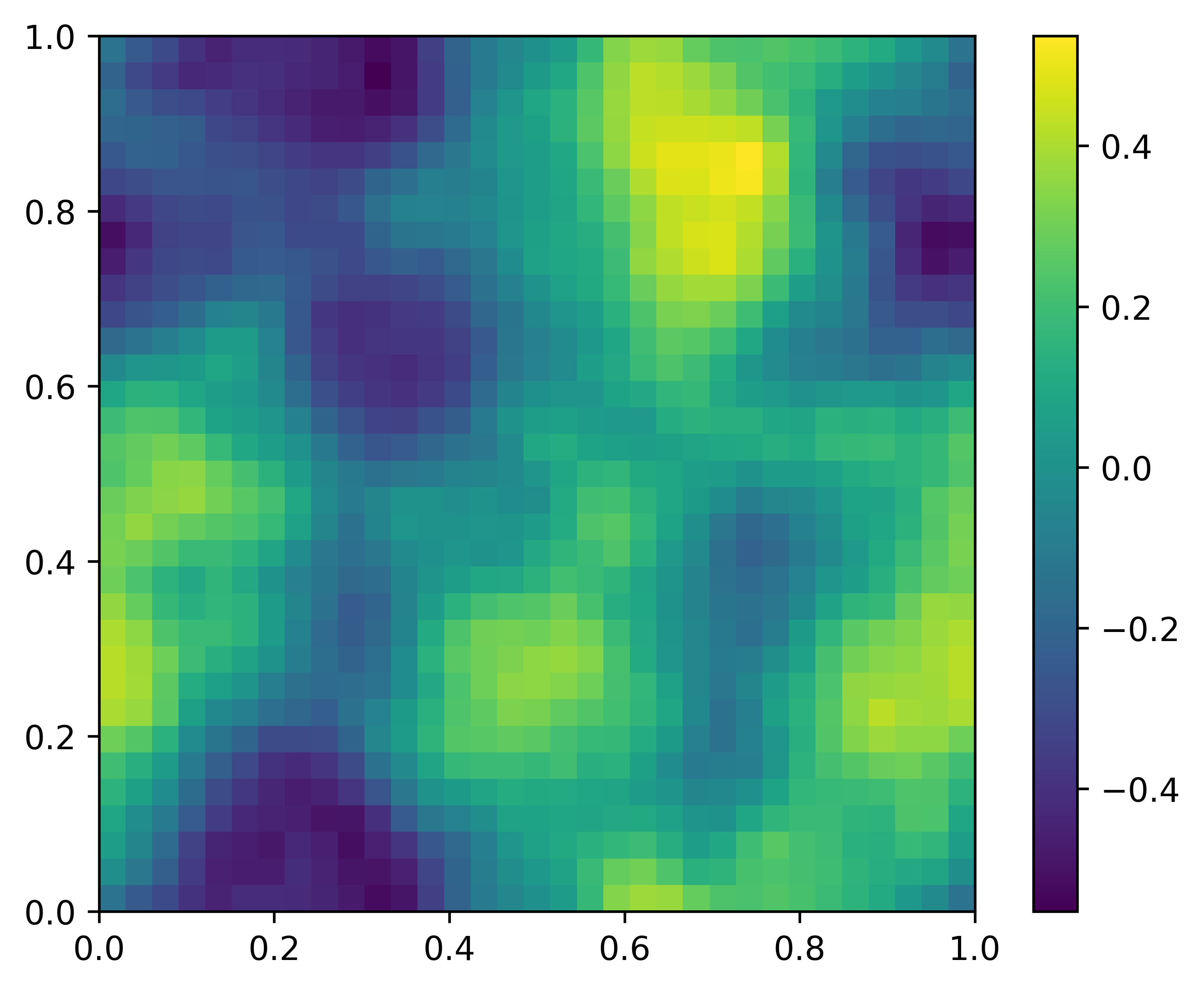}
\caption{Input}
\end{subfigure}
\begin{subfigure}{0.45\textwidth}
\includegraphics[width=\linewidth]{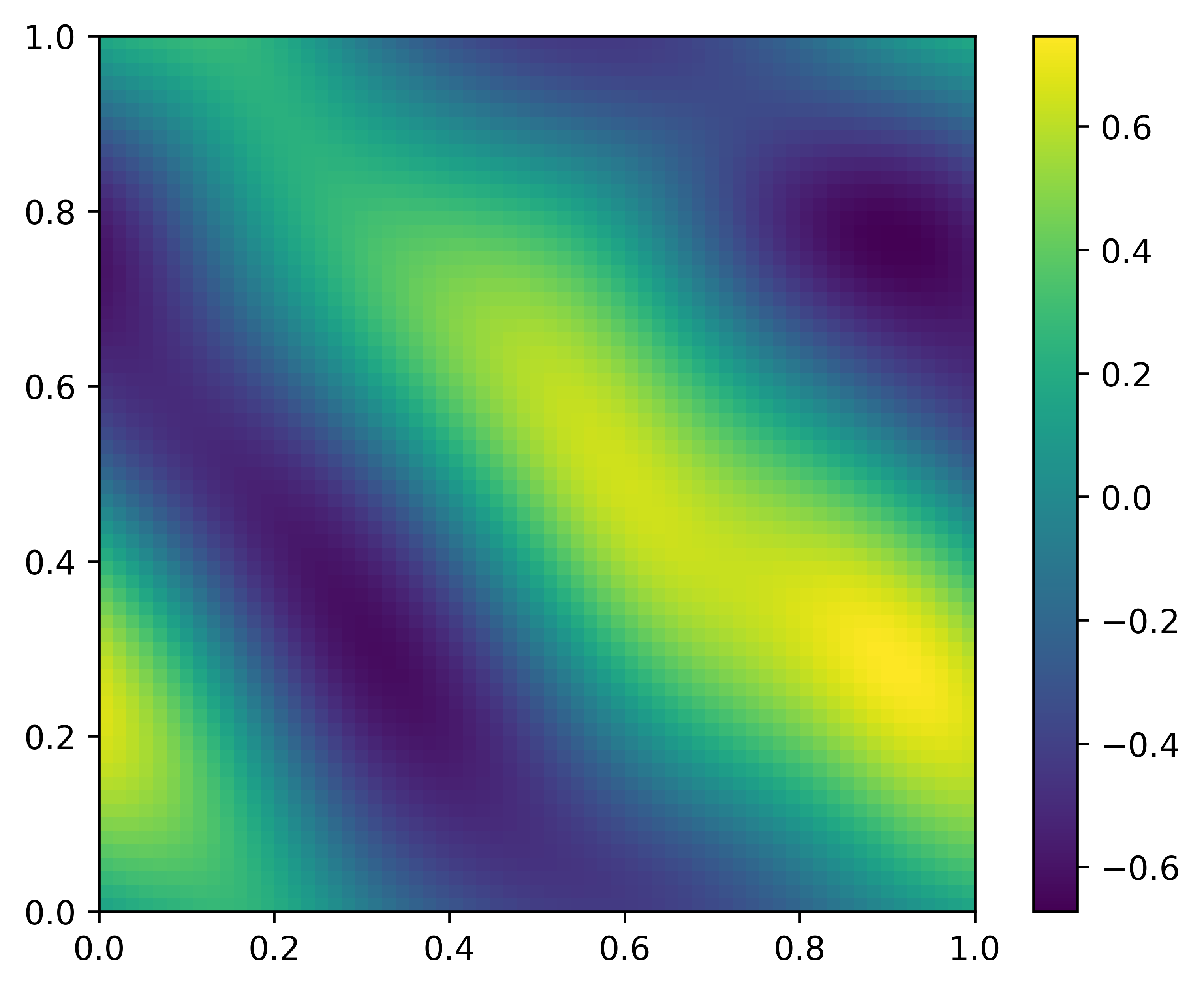}
\caption{Output}
\end{subfigure}
\caption{A sample illustrating the operator for the Navier-Stokes equations \eqref{eq:NS}. The input to the operator is given by the initial vorticity and the output is the vorticity at time $T=5$.}
\label{fig:nsex}
\end{figure}

\end{document}